\documentclass{article} 
\usepackage{iclr2027_conference,times}

\usepackage{amsmath,amsfonts,bm}

\def\eqref#1{equation~\ref{#1}}

\def\1{\bm{1}}

\def\va{{\bm{a}}}

\def\vv{{\bm{v}}}

\def\vx{{\bm{x}}}
\def\vy{{\bm{y}}}
\def\vz{{\bm{z}}}

\def\mA{{\bm{A}}}

\def\mI{{\bm{I}}}

\DeclareMathAlphabet{\mathsfit}{\encodingdefault}{\sfdefault}{m}{sl}
\SetMathAlphabet{\mathsfit}{bold}{\encodingdefault}{\sfdefault}{bx}{n}

\usepackage{hyperref}
\usepackage{url}
\usepackage{float}
\usepackage{amssymb}
\usepackage{soul}
\usepackage{graphicx}
\usepackage{xcolor}
\usepackage{algorithm}
\usepackage{algorithmic}

\usepackage{xspace}
\newcommand{\methodname}{PIVOT\xspace}
\title{
When the Prior Is Wrong: Probing Misspecified Diffusion Priors in Inverse Problems
}

\author{Frederic Wang \& Katherine L. Bouman \\
Department of Computing and Mathematical Sciences\\
Caltech\\
Pasadena, CA 91125, USA \\
\texttt{\{fzwang, klbouman\}@caltech.edu} }

\iclrfinalcopy 
\begin{document}

\maketitle

\begin{abstract}
  While diffusion priors generate high-quality posterior samples across many scientific inverse problems, they are often trained on purely simulated data, thus inheriting the errors and biases of these underlying sources. Many approaches to adapting diffusion priors from a single observation require sufficient conflict between the data and the prior to successfully adapt and do not allow specific misspecification hypotheses to be tested. On the other hand, existing out-of-distribution (OOD) detection methods determine whether the observation is compatible with a fixed prior, rather than specifying which assumptions to relax or by how much. To address these issues, we propose \methodname: Prior Inference Via Observation-driven Transformations, which formulates diffusion prior misspecification as inference over an underlying latent variable representing a scientifically plausible discrepancy. By expanding the prior along these plausible directions of misspecification, \methodname infers a posterior over the strength of each proposed discrepancy given the observation. This process can be efficiently amortized with just one posterior sampling pass using a conditional diffusion model. We also demonstrate how to determine if a proposed discrepancy  is inconclusive because the measurement is too noisy or ill-posed.  
 We validate \methodname  on seismic imaging and real-world black hole imaging problems,  where we find that OOD observations generally lead to larger perturbation strengths when the proposed discrepancy family captures relevant components of the mismatch, as long as the discrepancy is visible to the measurement. \methodname thereby provides an interpretable probe into how real-world priors could be misspecified.

\end{abstract}
\vspace{-0.2in}
\section{Introduction}

In ill-posed imaging inverse problems, our reconstructions are only as good as the assumptions we make. These assumptions, often encoded as prior probability distributions $p(\vx)$, describe what we believe plausible image solutions should look like before seeing the data. Yet in practice, the true prior distribution that governs nature’s variability is never known. Instead, we rely on approximations: models that are convenient, interpretable, or empirically motivated, but almost certainly wrong in subtle ways. These mismatches can bias the resulting reconstructions, leading us to underestimate uncertainty or to impose features that reflect our model more than reality. 

Diffusion models \citep{song2020score, ho2020denoising} are state-of-the-art data-driven priors for imaging inverse problems, and most existing posterior sampling methods \citep{chung2022diffusion,coeurdoux2024plug, zhu2023denoising, wu2024principled, zhang2025improving} treat the pretrained diffusion models as properly specified. However, many diffusion models used for scientific priors are trained on images from simulations or highly limited training sets, inheriting potentially incorrect assumptions about the world, simulation constraints, missing coverage of rare phenomena, and other biases.

Robust Bayesian analysis provides a natural framework for studying prior misspecification by asking how posterior conclusions change under perturbations of the prior \citep{berger1994overview}. For analytic priors, these perturbations can often be expressed through hyperparameters or explicit prior expansions \citep{nott2021using}. Extending this idea to data-driven priors is less straightforward: their assumptions are encoded implicitly in a high-dimensional training distribution, and generic perturbations need not preserve realistic image structure. For example, tempering a high-dimensional data-driven prior can shift probability mass away from its typical set, producing unrealistic samples; see Appendix \ref{sec:tempering-bad}. This motivates considering perturbations that correspond to plausible ways the learned prior could be misspecified. Data augmentation during training \citep{karras2022elucidating, zhang2025training} can be viewed as a form of prior expansion, but does not probe specific misspecification directions nor provide a way to determine if the expansion is detectable by our observations.


One line of work attempts to quantify how OOD a measurement is with respect to a given diffusion prior, using statistics such as evidence \citep{wang2026sample}, conditional evidence \citep{sprunck2025bayesian}, and KL divergence \citep{shoushtari2025unsupervised}. While these approaches are suitable for diagnosing OOD, they do not provide any insight as to how a mismatched prior distribution can be perturbed to best fit our observation. On the other hand, test-time diffusion prior adaptation methods finetune the network weights during sampling to better fit the measurements \citep{barbano2025steerable, chung2024deep, hu2024patch}. These methods require careful tuning to prevent overfitting to the measurement, and perform best when there is conflict between the prior and the observation. However, after finetuning, the diffusion model may no longer represent the score of a coherent prior, and the adaptation is driven largely by inductive biases in the network architecture rather than by an explicitly specified, scientifically plausible direction of misspecification, making the resulting reconstructions difficult to interpret. Patch-based diffusion priors have been shown to perform better on OOD measurements \citep{hu2024learning, hu2024patch} by weakening global correlations in images; however, patch size controls how strongly the original prior is relaxed, and existing methods do not infer this.

\textbf{Our contributions:} We propose \methodname: Prior Inference Via Observation-driven Transformations, a framework for probing different hypotheses of prior misspecification from a single observation. \methodname expands the base diffusion prior with a hierarchical family of perturbed priors, and computes the posterior over perturbation strengths supported by the observation according to the Bayesian evidence. This discrepancy family, i.e. an explicit perturbation retroactively applied to simulated data or an implicit simulator variable, encodes a plausible scientific hypothesis and provides a probe into how likely these discrepancies are given our observation relative to synthetic in-distribution and OOD observations. Rather than performing posterior sampling separately for many different perturbation strengths, \methodname efficiently amortizes this using a conditional diffusion prior trained with label dropout. While some discrepancies are either invisible or confounded under evidence (a fundamental limitation arising from the forward operator used to collect observations), we find that when the data likelihood is sensitive to the perturbation family, then the computed evidence recovers whether the proposed discrepancy is supported by the observation and at what strength. While a high discrepancy strength does not establish it as the unique discrepancy explaining the prior mismatch, \methodname enables the testing of specific, physically motivated hypotheses. We also find that empirically, the reconstruction quality under the expanded prior can exceed the strong prior, even if the true discrepancy is not captured by the family. In doing so, \methodname reframes prior selection not as a fixed choice but as an inference problem in its own right.




We validate \methodname on a toy Gaussian model where the evidence and diffusion score can be computed analytically. Here, we find that despite the approximate diffusion posterior sampling and approximate density evaluation, 
we still achieve highly accurate perturbation strength estimates. 
We also demonstrate an example on the travel-time tomography seismic imaging problem, where we design a discrepancy family of elliptical subsurface structures, inspired by magma chambers, to extend a narrow data-driven OpenFWI prior \citep{deng2022openfwi}. 
We find that our method identifies additional elliptical subsurface features when they are present, and favors such structures for completely OOD measurements when they better explain the observations.

Finally, we use \methodname on M87* black hole imaging, where in the case of real data the true prior is unknown. Priors trained on physical simulations of black hole images~\citep{mizuno2022grmhd} have been successfully used to reconstruct real M87* observations~\citep{wu2024principled, feng2024event}. To demonstrate how our approach can probe different hypotheses on real scientific data, we design two discrepancy families: additional compact emission and elliptical deformations. We find that the black hole forward model is not very sensitive to elliptical deformations, and that the real M87* observations give inconclusive evidence for ellipticity. In contrast, the real M87* observations on April 5 and 6, 2017 \citep{dataset} demonstrate consistently larger compact-emission discrepancies than most in-distribution observations. This feature is consistently located southwest of the ring across all four observation days, suggesting the presence of additional compact emission that is not reconstructed under the strong simulation prior. This emission is one explanation that reduces the disagreement between the M87* observations and the GRMHD prior, and could be associated with material flowing outward from the black hole in a jet.

\section{Background}

\subsection{Diffusion Inverse Problem Solvers}
Given a ground truth image or signal $\vx$, forward operator $A(\cdot)$, and measurement noise $\bm{\varepsilon}$, we can obtain a measurement $\vy = A(\vx) + \bm{\varepsilon}$. Classical inverse problem solvers use handcrafted prior assumptions, such as total variation or $\ell_1$ norm, to narrow down the ill-posedness of the problem. Diffusion models can be used as data-driven priors for posterior sampling. Early approaches to diffusion posterior sampling \citep{chung2022diffusion, song2023pseudoinverse, boys2023tweedie} guide the sampling trajectory with approximations of the intractable $\nabla_{\vx} \log p(\vy \mid \vx_t)$. Recently, posterior samplers that anneal along specified marginals using invariant transition kernels \citep{zhu2023denoising, wu2024principled, zhang2025improving, zhu2409think}  have led to more accurate posterior samples. Inference-time steering of diffusion models \citep{uehara2025inference} can also be viewed as solving an inverse problem, where the reward or steering function is interpreted as a Bayesian likelihood. Many of these methods use Sequential Monte Carlo or related particle-based methods to reweight or steer diffusion trajectories towards a target distribution \citep{he2025rne,singhal2025general, skreta2025feynman, wang2026simple, ren2026driftlite, chen2025solving}.

\subsection{Conditional Diffusion Models}
Given a condition $\va$ and a conditioned image $\vx_0 \sim p(\vx_0 \mid \va)$, the variance exploding (VE) forward diffusion process at time $t \in [0, T]$ and noise schedule $\sigma_t$ is given by:
\begin{align}
    \vx_t = \vx_0 + \sigma_t \vz_t, \quad \vz_t \sim \mathcal{N}(0, \mI)
\end{align}
We can sample from the diffusion by discretizing the 
probability flow ODE (PF-ODE, Eq. \ref{eq:pf-ode}):
\begin{align}
     d \vx_t
&= -\sigma_t \sigma_t' \nabla_{\vx_t}\log p(\vx_t \mid \va) dt \label{eq:pf-ode}
\end{align}
where  
$\nabla_{\vx_t} \log p(\vx_t \mid \va)$ is learned. We use the VE process and PF-ODE throughout this paper for ease of notation, but our method can be extended to other diffusion noise schedules or the reverse SDE formulation sampling under a reparameterization. Given a noisy image $\vx_t \sim p(\vx_t \mid \va)$, Tweedie's formula gives us the posterior mean: 
\begin{align}
    \mathbb{E}[\vx_0 \mid \vx_t, \va]
= \vx_t + \sigma_t^2 \, \nabla_{\vx_t} \log p(\vx_t \mid \va).
\end{align}

To strengthen the conditioning signal, Bernoulli dropout with a null condition $\varnothing$ is used during training \citep{ho2022classifier}; thus, the null condition is trained on the marginalized distribution:
\begin{align}
    p(\vx \mid \varnothing) = \int p(\vx \mid \va) p(\va) d\va\label{eq:unconditional-head}
\end{align}

The density of a diffusion model can be evaluated (up to integration error) by integrating along the PF-ODE \citep{song2021maximum}, but this approach is typically compute intensive. More recently, Radon-Nikodym estimators \citep{he2025rne} use path density ratios to estimate diffusion densities and can also be used for inference-time steering. The evidence lower bound (ELBO) \citep{song2021maximum} also provides an efficient, differentiable surrogate log-density:
\begin{align}\label{eq:ELBO}
L_{ELBO}(\vx_0 \mid \va) &= 
-
\int_{0}^{T}
\mathbb E_{\bm{\varepsilon}\sim\mathcal N(0,\mI)}
\left[
\frac{\dot{\sigma}_t}{\sigma_t^3}
\left\|
D_\theta(\vx_0+\sigma_t\bm{\varepsilon},\sigma_t,\va)-\vx_0
\right\|_2^2
\right] dt \\
&\leq  \log p(\vx_0 \mid \va) - C(\vx_0)
\end{align}

As $C(\vx_0)$ is independent of the condition $\va$, the ELBO has been used for classification \citep{li2023your} and for providing accurate gradients for sampling conditions of clean images \citep{mudur2024diffusion}, despite the remaining lower bound gap.

\subsection{Diffusion OOD adaptation}

When a prior $p(\vx)$ is mismatched to an observation $\vy$, one approach is to relax the prior rather than use it as-is. Classical robust Bayesian methods do this by, for example, mixing a strong prior with a weaker contamination prior \citep{berger1986robust}, or introducing hyperparameters that allow the prior to vary over different assumptions \citep{tiede2026hierarchical}. Recent work has explored analogous strategies for diffusion priors. Deep Diffusion Image Prior (DDIP) methods \citep{barbano2025steerable, chung2024deep} adapt a pretrained diffusion model to a single observation at test time by fine-tuning the denoiser as a Deep Image Prior \citep{ulyanov2018deep}. Alternatively, patch-based diffusion priors relax global image structure by modeling local patches, improving robustness to distribution shift \citep{hu2024learning}, and can also be combined with DDIP \citep{hu2024patch}. DDIP requires careful tuning to avoid overfitting and adapts the prior in directions determined implicitly by the network, while patch-based methods can only encode the specific hypothesis that global image structure is the source of mismatch.

Another related line of work uses a large bank of measurements with different forward operators to train or finetune a diffusion model. Ambient diffusion methods \citep{daras2023ambient, aali2024ambient, daras2024consistent} train a diffusion model using many corrupted observations, while expectation-maximization methods finetune an initial diffusion model prior \citep{bai2024expectation,rozet2024learning,barco2025tackling}. However, these works assume many observations while we target situations where only a single observation may be available. EM methods have been used to optimize the input condition in latent diffusion models for solving inverse problems \citep{chung2023prompt}, offering an alternative approach of adapting the prior for large pretrained latent diffusion models.

\vspace{-0.1in}
\section{Method}
\vspace{-0.1in}
\methodname has three steps. First, we define a family of priors \(p(\vx\mid \va)\), where $\va=0$ is the original prior and increasing $\va$ introduces more of a user-specified discrepancy. Second, we train a conditional diffusion model over this family, with label dropout so that the unconditional dropout condition represents the family marginalized over \(\va\). Finally, for a new observation \(\vy\), we perform posterior sampling once under this marginalized prior and use the resulting samples and conditional diffusion model to estimate which discrepancy strengths \(\va\) are best supported by the observation.
\vspace{-0.05in}
\subsection{Posterior Sampling over an Expanded Prior Family}
\vspace{-0.05in}
\begin{figure}
    \centering
    \includegraphics[width=0.7\linewidth]{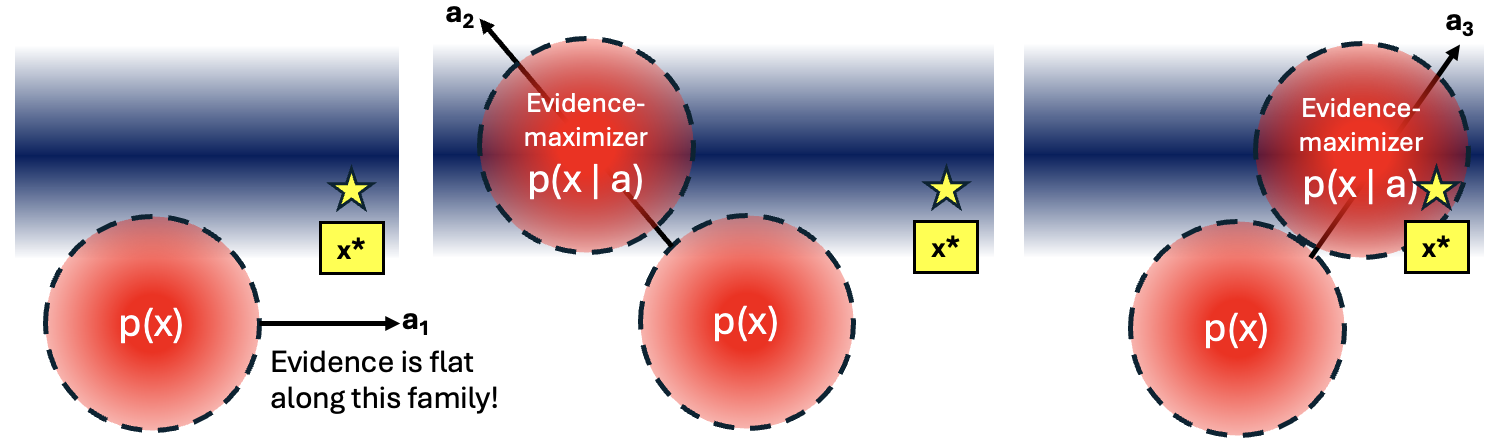}
    \caption{Three examples of maximizing evidence, where each perturbation family $i$ probes the mean of the Gaussian prior $p(\vx)$ along $\va_i$. \textbf{Left:} The perturbation may be useful, but is unsupported by the measurement, so the evidence is flat; with one measurement, this is unavoidable. \textbf{Middle:} A failure case. The perturbation family is poor, so the evidence is maximized at a prior that the ground truth $\vx^*$ is less likely under. Therefore, the perturbation family which determines the coupling between measurement and null space must be selected well a priori and calibrated. \textbf{Right:} The perturbation family is specified well, and the evidence-maximizing prior in this family captures $\vx^*$.}
    \label{fig:failure-cases}
\end{figure}

We begin by specifying a plausible way in which the base prior could be misspecified. Importantly, this discrepancy family must be chosen {\it a priori} to represent a scientifically meaningful hypothesis. Bayesian evidence can determine whether a perturbation improves agreement with the measurements, but it cannot determine whether the associated changes in the unobserved null space are physically meaningful. Restricting the perturbation to a scientifically plausible direction therefore helps ensure that improvements in measurement fit are accompanied by changes we are willing to believe in the unconstrained parts of the solution. Figure~\ref{fig:failure-cases} (middle) illustrates what can happen when this direction is poorly chosen.

Let $\va\in[0,1]^d$ denote the strengths of $d$ different discrepancies, with $\va=\bm{0}$ as the
base prior and $a_i=1$ as the maximum perturbation considered for discrepancy $i$. Given a sample
$\vx' \sim p_0(\vx)$ from the base prior $p_0$, let
$p(\vx\mid \vx',\va)$ represent a distribution over perturbations of \(\vx'\) with strength \(\va\), such that the non-perturbed case gives $p(\vx \mid \vx', \bm{0}) = \delta(\vx - \vx')$. This defines the family of priors: $ p(\vx\mid \va) = \int p(\vx\mid \vx',\va)p_0(\vx')\,d\vx'$. We train a conditional diffusion model on samples from this family. During
training, we sample $\va\sim\mathrm{Unif}([0,1]^d)$ and use label dropout so that the
null condition $\varnothing$ learns the prior marginalized over discrepancy
strength $\va$:
\begin{equation}
    p(\vx\mid\varnothing)
    =
    \int p(\vx\mid \va)p(\va\mid\varnothing )\,d\va,
    \qquad
    p(\va\mid\varnothing)=\mathrm{Unif}([0,1]^d).
\end{equation}
The conditional model therefore represents each individual prior $p(\vx\mid \va)$,
while the unconditional head represents the entire discrepancy family at once.
Although we
use a uniform distribution over $\va$ during training, different hyperpriors
$p(\va)$ can be imposed at inference time through reweighting, as described
in Section \ref{sec:determining-strength}. See Algorithm \ref{alg:spikeslab-conditional} for an example training iteration.

\begin{algorithm}[H]
\caption{One training iteration of conditional diffusion}\label{alg:spikeslab-conditional}
\begin{algorithmic}[1]
\REQUIRE Base prior $p_0(\vx)$, perturbation kernel $p(\vx \mid \vx', \va)$, label dropout $p_{dropout}$, conditional diffusion net $D_\theta(\vx_t, \sigma_t, \va)$, diffusion loss weighting $w(\sigma_t)$.
\STATE Sample $\va \sim Unif([0,1]^d)$, $\vx' \sim p_0(\vx)$, $\vx \sim p(\vx\mid \vx', \va)$.
\STATE Sample $t$ from noise schedule, $\bm{\varepsilon} \sim \mathcal{N}(0,\mI)$, $\vx_t \gets \vx + \sigma_t \bm{\varepsilon}$
\IF{rand $<p_{dropout}$}
\STATE Take a gradient step on the diffusion objective $w(\sigma_t) \left\|
D_\theta(\vx_t,\sigma_t,\varnothing)-\vx
\right\|_2^2$
\ELSE
\STATE Take a gradient step on the diffusion objective $w(\sigma_t) \left\|
D_\theta(\vx_t,\sigma_t,\va)-\vx
\right\|_2^2$
\ENDIF
\end{algorithmic}
\end{algorithm}

The unconditional head represents the prior marginalized over all discrepancy strengths. Therefore, conditioning it on an observation $\vy$, whose likelihood is independent of the condition $\va$, gives
\begin{align}
    p(\vx \mid \vy, \varnothing) = \int p(\vx \mid \vy, \varnothing, \va) p(\va \mid \vy, \varnothing) d\va = \int p(\vx \mid \vy, \va) p(\va \mid \vy, \varnothing) d\va 
\end{align}
The marginalized posterior combines the posterior associated with each discrepancy strength $\va$, weighted by how well that prior explains the observation. Since we assume a uniform prior over discrepancy strengths, the posterior weight assigned to each $\va$ is proportional to its evidence $p(\vy \mid \va)$. Thus, discrepancy strengths that make the observed measurement more likely receive more weight in the marginalized posterior. This allows a single posterior sampling run to explore the full discrepancy family, rather than separately solving the inverse problem for each value of $\va$. To perform posterior sampling, we use the DPMC sampler \citep{zhu2409think} which anneals along $p(\vx_t) p(\vy \mid \mathbb{E}[\vx_0 \mid \vx_t, \varnothing])^{\beta_t}$ for $t \to 0$ in a predictor-corrector fashion; this sampler is asymptotically exact, allows for high-noise mixing and mode exploration, validated on an analytic Gaussian mixture in Section \ref{sec:gaussian-ablation}, and demonstrates strong empirical performance  for scientific inverse problems in Section \ref{sec:travel-time-tomography} and \ref{sec:black-hole-imaging}.

\vspace{-0.1in}
\subsection{Inferring the Discrepancy Posterior}\label{sec:determining-strength}
\vspace{-0.1in}
Given samples from the marginalized posterior $p(\vx\mid \vy,\varnothing)$, we now
want to determine which discrepancy strengths $\va$ are supported by the
observation. The key observation is that each posterior sample $\vx$ contains
information about which member of the prior family $p(\vx\mid \va)$ could have
generated it. Averaging this information across posterior samples allows us to
recover the discrepancy posterior $p(\va\mid \vy)$ without rerunning posterior
sampling for each value of $\va$.

\begin{align}
    p (\va \mid \vy) &\propto p(\va) p(\vy\mid \va)
    \\
    &\propto p(\va) p(\va\mid \vy, \varnothing) \\
    &\propto p(\va)
    \int p(\va \mid \vx, \vy, \varnothing) p(\vx \mid \vy, \varnothing) d\vx\\
    &= p(\va) \mathbb{E}_{\vx \sim p(\vx \mid \vy, \varnothing)}[p(\va \mid \vx, \varnothing)]  \label{eq:mcmc-eq}
\end{align}
The second line follows because the training distribution
$p(\va\mid\varnothing)$ is uniform. For each posterior sample $\vx$, we use the conditional prior $p(\vx \mid \va)$ to determine which discrepancy strengths are most compatible with that sample, yielding $p(\va \mid \vx, \varnothing)$. Averaging these distributions across posterior samples then gives $p(\va \mid \vy)$. We consider two ways to estimate this distribution: MCMC sampling over $\va$, or
direct evaluation on a grid if $\va$ is low-dimensional ($d=1,2$). Using the MCMC approach, for each posterior sample $\vx \sim p(\vx \mid \vy, \varnothing)$, we can sample $\va$ using reflected Langevin dynamics and score $\nabla_\va \log p(\va\mid \vx,\varnothing)
    = \nabla_\va \log p(\vx\mid \va)$. The result can be reweighted by the desired $p(\va)$ and combined across all samples of $p(\vx \mid \vy, \varnothing)$. In practice, we approximate $\log p(\vx\mid \va)$ using the discretized ELBO
(Eq.~\ref{eq:ELBO}); we directly validate this approximation and show it recovers the analytic posterior in Section \ref{sec:gaussian-ablation} and the true discrepancy strength in Appendix \ref{sec:ddip-analysis} despite the lower bound gap. Alternatively, for low-dimensional $\va$, we can evaluate its (unnormalized) posterior on a
grid:

 \begin{align}
    p (\va \mid \vy) &\propto p(\va) \mathbb{E}_{\vx \sim p(\vx \mid \vy, \varnothing)}[p(\va \mid \vx, \varnothing)]\\
    &= p(\va) \mathbb{E}_{\vx \sim p(\vx \mid \vy, \varnothing)}[\frac{p(\vx \mid \va)}{ p(\vx \mid \varnothing)}] \\
    &= p(\va) \mathbb{E}_{\vx \sim p(\vx \mid \vy, \varnothing)}[\frac{p(\vx \mid \va)}{ \mathbb{E}_{\va \sim Unif([0,1]^d)}[p(\vx \mid \va)]}] \label{eq:grid-density-eq}
\end{align}
where the numerator can be computed with ELBO for condition $\va$. The denominator can be computed by averaging over the rest of the grid, which gives a lower variance estimate than performing a second ELBO on the unconditional head as demonstrated in Section \ref{sec:gaussian-ablation}. Note that the constant $C(\vx)$ from Eq. \ref{eq:ELBO} appears in both the numerator and denominator after the ELBO and cancels out. Importantly, our method relies on the ELBO approximation and assumes that the unconditional head $p(x \mid \varnothing)$ matches the true mixture $\int p(x \mid a) p(a) da$ despite finite model capacity. We perform additional validation using our learned diffusion priors in Appendix \ref{sec:elbo-validation} where we find that the $p(a \mid x)$ obtained by the ELBO closely matches that of the integrated PF-ODE density, and the unconditional head represents the true mixture well.

Finally, $p(\va\mid \vy)$ should not be interpreted in isolation, as it
depends on how strongly the measurement operator constrains the proposed
discrepancy (Fig. \ref{fig:failure-cases}, left or middle). We therefore compare the inferred discrepancy posterior against
those obtained from synthetic in-distribution, in-family perturbed, and
OOD observations to calibrate what discrepancies can
meaningfully be distinguished. 

\vspace{-0.1in}
\section{Experiments}
\vspace{-0.1in}
We evaluate \methodname on an analytic Gaussian, travel time tomography, and black hole imaging. We use $d=1$ in our experiments to isolate scientific hypotheses and make calibration straightforward; a $d=2$ ablation is shown in Appendix \ref{sec:gaussian-ablation-2d}. Experimental details are described in Appendix \ref{sec:experimental-details}. 

\vspace{-0.1in}
\subsection{Gaussian Ablation}\label{sec:gaussian-ablation}
\vspace{-0.1in}
\begin{figure}
    \centering
    \includegraphics[width=0.9\linewidth]{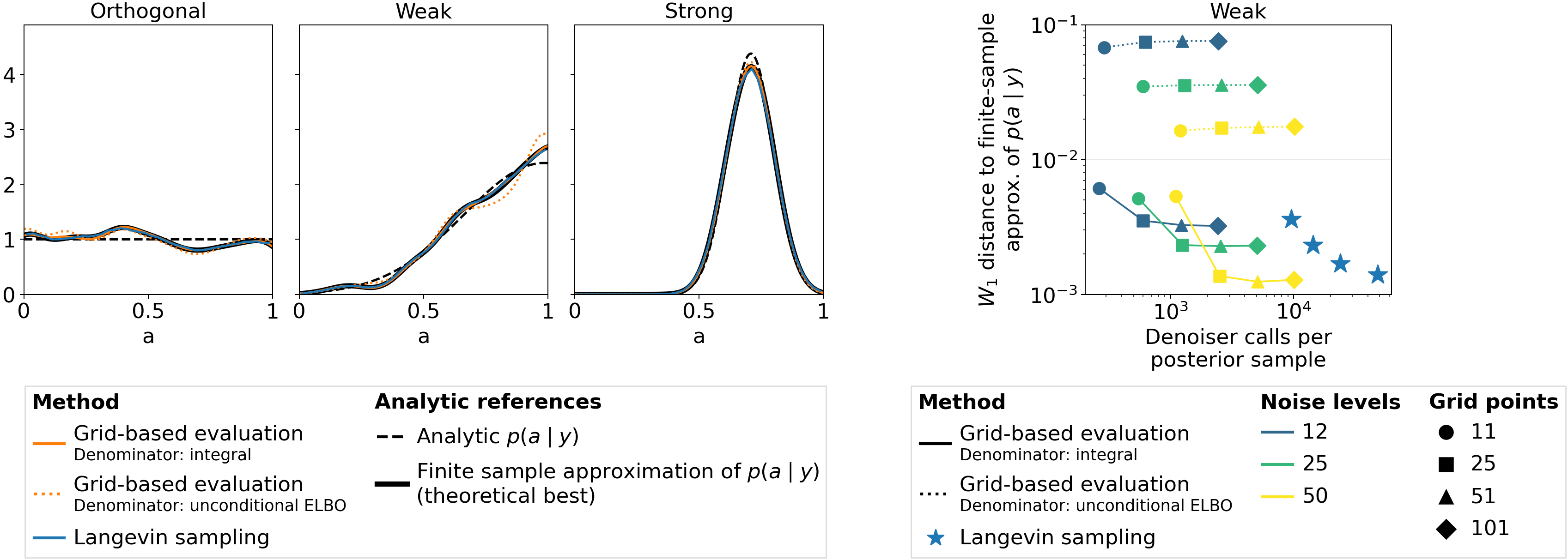}
    \caption{\textbf{(Left)} Discrepancy posterior $p(\va \mid \vy)$ results on a Gaussian toy experiment where the forward operator has an increasing level of sensitivity to the perturbation (orthogonal, weak, strong). The ideal analytic $p(\va \mid \vy)$,  analytic finite-sample (50) approximation of $p(\va \mid \vy)$ which is the theoretical best we can achieve, and ELBO-based discrepancy methods are shown. The shape of the evidence curve is correctly recovered up to the Monte Carlo noise of the finite-sample approximation despite the approximate sampler and ELBO. \textbf{(Right)} computation vs. accuracy plot. With $O(10^3)$ denoiser calls per sample, the analytic finite-sample $p(\va \mid \vy)$ is achieved with $O(10^{-3})$ $W_1$ distance.}
    \label{fig:gaussian-toy}
\end{figure}

\begin{figure}
    \centering
    \includegraphics[width=0.9\linewidth]{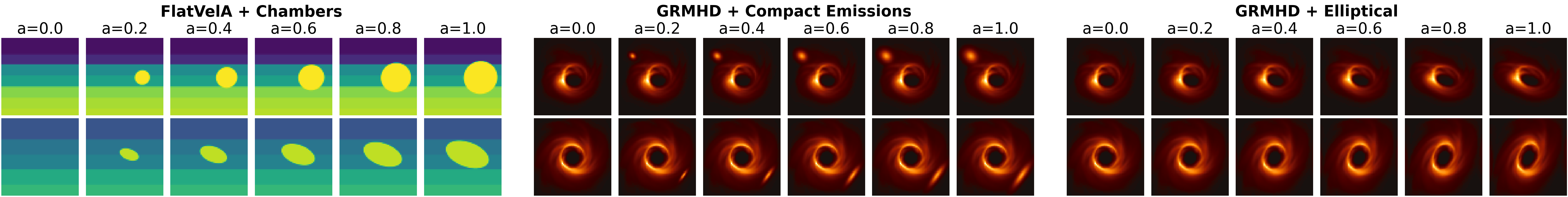}
    \caption{Two prior samples from each discrepancy family at different perturbation (i.e., $\va$) levels.}
    \label{fig:unconditional}
\end{figure}

We first validate \methodname on a multimodal 100-component Gaussian mixture prior and linear inverse problem $\vy = \mA \vx + \bm{\varepsilon}$, $\mA \in \mathbb{R}^{m \times n}$, $\bm{\varepsilon} \sim \mathcal{N}(0, \sigma_y^2 \mI)$ with $m=250$, $n=1000$ and $\sigma_y = 0.1$, whose analytic evidence can help inform us how many posterior samples and ELBO evaluations are necessary to accurately model $p(\va \mid \vy)$ for high-dimensional images. We consider a discrepancy family where the mean of each Gaussian component is shifted independently in a way that is (1) orthogonal to the forward operator, so evidence should be flat (Fig. \ref{fig:failure-cases}, left), (2) weakly constrained by the forward operator,  and (3) strongly constrained (Fig. \ref{fig:failure-cases}, middle and right). We generate posterior samples using the analytic diffusion score of the Gaussian mixture $\nabla_\vx \log p(\vx_t)$ and the DPMC posterior sampler, which while asymptotically exact, accounts for finite-time sampling errors. 

We perform the ablation sweep for both MCMC (Eq. \ref{eq:mcmc-eq}) via reflected Langevin, and the 1D grid-based evaluation (Eq. \ref{eq:grid-density-eq}) using both the unconditional ELBO and grid-based integration to evaluate the denominator. Results are shown in Fig. \ref{fig:gaussian-toy}. The shape of analytic $p(\va \mid \vy)$ is recovered up to the finite-sample Monte Carlo error with $O(10^3)$ denoiser calls per sample, indicating that the finite-time DPMC error and ELBO gap are small in this regime. Furthermore, using the integrated grid to compute the denominator of Eq. \ref{eq:grid-density-eq} results in lower error than ELBO on $p(\vx \mid \varnothing)$.

\vspace{-0.1in}

\subsection{Travel Time Tomography}\label{sec:travel-time-tomography}
\vspace{-0.1in}

\begin{figure}
    \centering
    \includegraphics[width=0.8\linewidth]{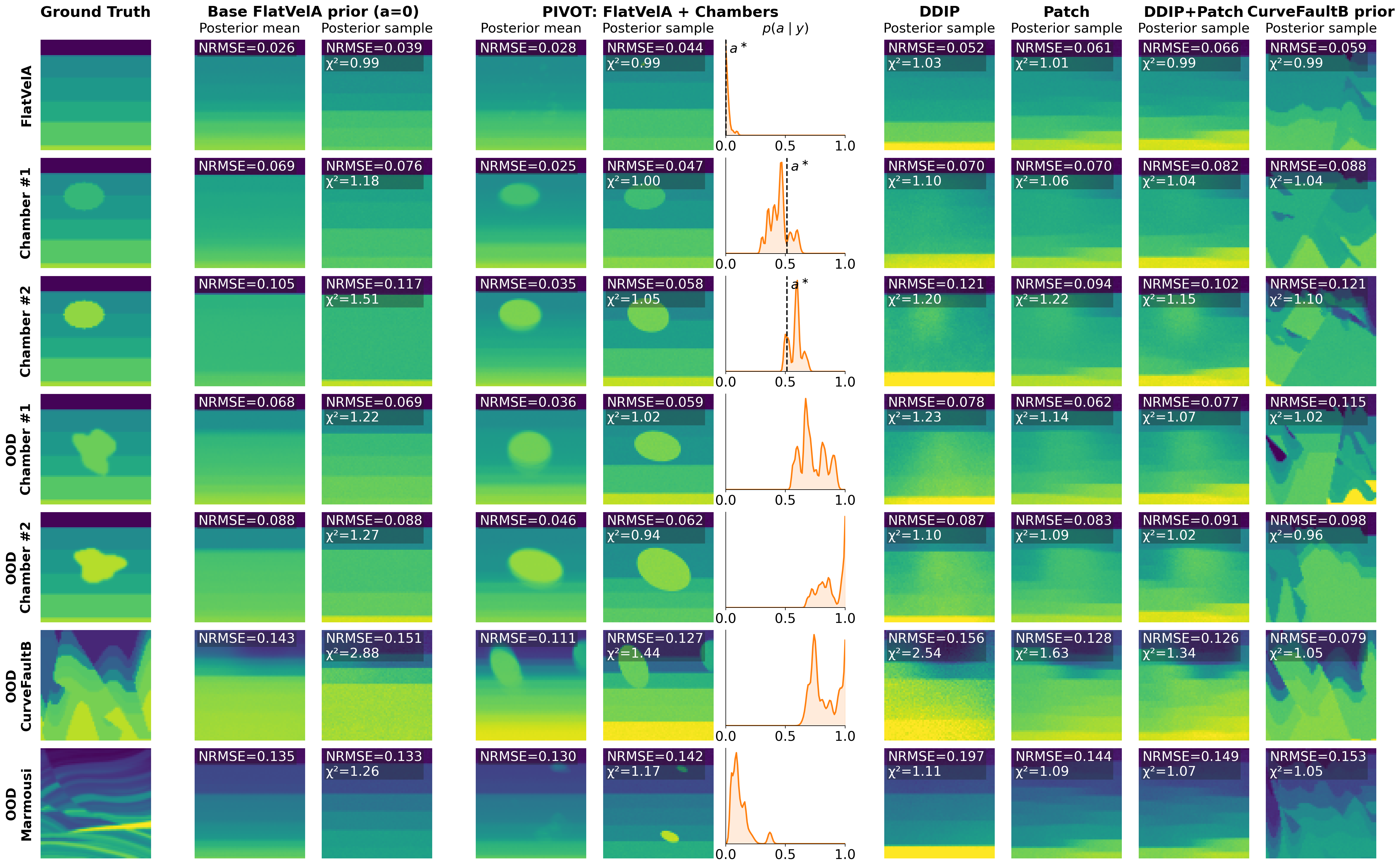}
    \caption{
    Posterior mean and samples of FlatVelA (in-distribution), Chamber 1 and 2 (in-family), and OOD measurements. For the in-family measurements, \methodname detects an elliptical feature and places substantial posterior mass near the true discrepancy $\va^*$, while detecting similar subsurface features of OOD measurements. The adaptation baselines (DDIP, DDIP+Patch) are not able to distinguish the chamber from a continuous blur, while the CurveFaultB prior imposes extra  structure.}
    \label{fig:ttt-discrepancy}
\end{figure}

In seismic imaging, the arrival time of earthquakes at different receivers provides information about the structure under the Earth's surface. Given  earthquake location $s$ and receiver location $r$, the velocity field $V$ is related to travel-time $T$ by the eikonal equation $\|\nabla_r T(s,r)\|^2 = V(r)^{-2}$. Given measurements $y_{ij} = T(s_i, r_j) + \varepsilon_{ij}$, the goal is to recover the velocity field $V(r)$ at all locations $r$. We use the OpenFWI \citep{deng2022openfwi} FlatVelA dataset as a strong prior, which imposes discrete layers of increasing velocity; this provides a simplified setting where synthetic OOD velocity fields can be easily generated. A natural scientific question is whether the subsurface contains previously unmodeled structures, such as magma chambers, that are absent from the prior. To probe this, we design a family of chamber discrepancies (Fig. \ref{fig:unconditional}, left). We find that the observations are sensitive to the discrepancy: using the posterior mean $\mathbb{E}[\va \mid \vy]$ as a summary statistic, the in-distribution observations have low discrepancy means while the in-family discrepancy means are high (Fig. \ref{fig:ttt-calibration}, in Appendix), making this a case where the perturbation is observable (Fig. \ref{fig:failure-cases}, right).

We compare \methodname to adapted FlatVelA via DDIP \citep{chung2024deep}, FlatVelA patch prior \citep{hu2024learning}, adapted FlatVelA via DDIP+Patch \citep{hu2024patch}, and the more complex OpenFWI CurveFaultB prior in Fig. \ref{fig:ttt-discrepancy}. \methodname correctly determines that there are no additional chamber features for in-distribution measurements while  detecting the chamber for in-family measurements and accurately inferring the true discrepancy $\va^*$. DDIP and Patch detect some discrepancy when there is a chamber present, but blur the recovered feature into a diffuse structure; using the CurveFaultB prior imposes substantial additional geological features. For fully OOD CurveFaultB measurements, although the ground truth does not come from the chamber family, a few local features align with the elliptical shape; \methodname detects mismatch and adds those features. In contrast, for the OOD Marmousi measurement, the fine-scale structure is invisible to the forward operator and there are no chamber-like features, so the data provide little support for increasing the discrepancy strength.
\vspace{-0.1in}
\subsection{Case Study: M87* Black Hole Imaging}\label{sec:black-hole-imaging}
\vspace{-0.1in}
Black hole images cannot be observed directly; instead they must be inferred from sparse measurements in a highly ill-posed inverse problem whose solutions depend critically on prior assumptions. Physical models, such as General-Relativistic Magnetohydrodynamics (GRMHD) \citep{mizuno2022grmhd}, can be used to generate black hole simulations from which priors can be learned. However, these models reflect our current understanding of black hole physics, which is likely incorrect.


We obtain measurements from the Event Horizon Telescope network (EHT)~\citep{event2019first1}. Given true frequency-domain measurements $\vv^*$ of the image, we obtain measurements from pairs of telescopes $i$, $j$ through the forward model $v_{i,j} = g_i g_j e^{i(\phi_i - \phi_j)} v_{i,j}^* + \varepsilon_{i,j}$, where $g_i$ is a station-dependent gain error, $\phi_i$ is a station-dependent phase error from atmospheric noise and $\varepsilon_{i,j}$ is thermal noise. We define two quantities to decouple the noise sources into approximately independent Bayesian likelihoods: the closure phase $y_{cp} = \angle v_{i,j} v_{j,k} v_{k,i}$ for all minimal sets of 3 telescopes $(i,j,k)$,  and the closure amplitude $y_{ca} = \frac{|v_{i,j}| |v_{k,l}|}{|v_{i,k}| |v_{j,l}|}$ for all minimal sets of 4 telescopes  $(i,j,k,l)$ \citep{thompson2017interferometry, blackburn2020closure}. 
Two major open questions of the M87* black hole data collected in 2017 \citep{dataset} concern the ring ellipticity and the possible presence of additional emission to the southwest (i.e., bottom right) that could correspond to a jet feature~\citep{dahale2025origin, Saurabh2026M87Jet}. Most existing analysis \citep{event2019first6} computes the likelihood of these features under simplified priors; however, these lack the high-resolution features of GRMHD, which could lead to alternate conclusions. We introduce two families of perturbations to answer these questions: additional compact \textbf{emission} and elliptical \textbf{deformations} (Fig.  \ref{fig:unconditional}, middle and right).


\begin{figure}
    \centering
    \includegraphics[width=0.9\linewidth]{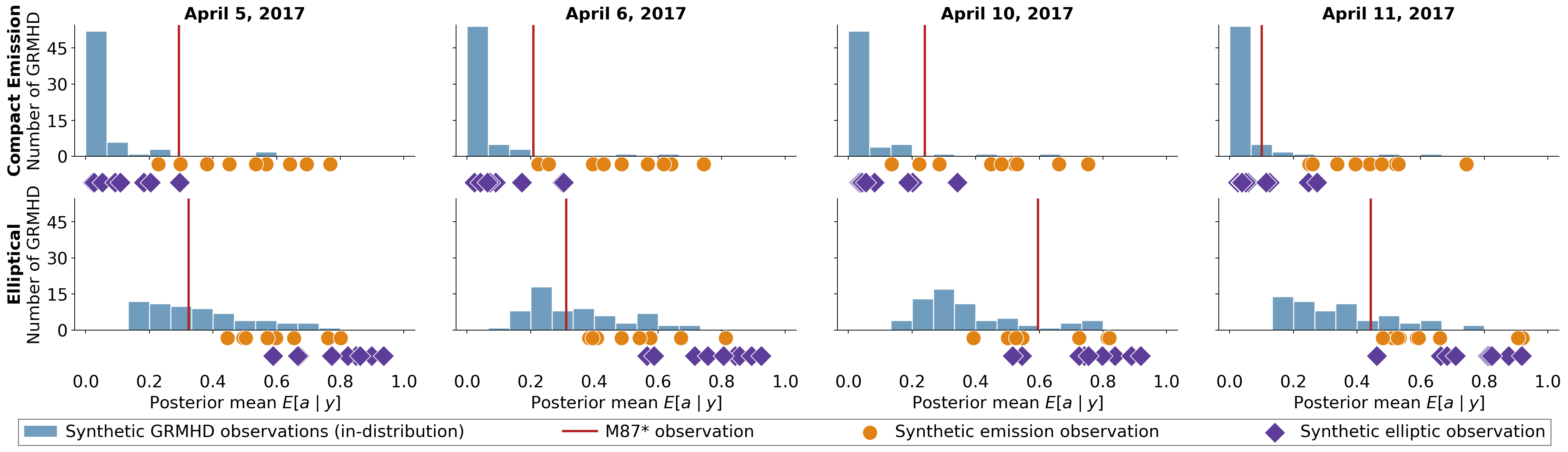}
    \caption{Discrepancy mean $\mathbb{E}[\va \mid \vy]$ for in-distribution synthetic GRMHD observations, in-family perturbed GRMHD observations, and OOD observations for the compact emission (\textbf{top}) and elliptic (\textbf{bottom}) discrepancies, using the M87* forward model on April 5, 6, 10, and 11. The forward model is sensitive to emission and accurately separates in-distribution and perturbed samples, but poorly constrains ellipticity, often assigning flat or broad posteriors to the elliptic discrepancy prior for in-distribution GRMHD observations. The M87* observation has a consistently high emission discrepancy mean compared to in-distribution GRMHD, suggesting a likely mismatch, while its elliptic discrepancy falls under standard variability of synthetic GRMHD observations.}
    \label{fig:calibration}
\end{figure}

\begin{figure}
    \centering
    \includegraphics[width=0.8\linewidth]{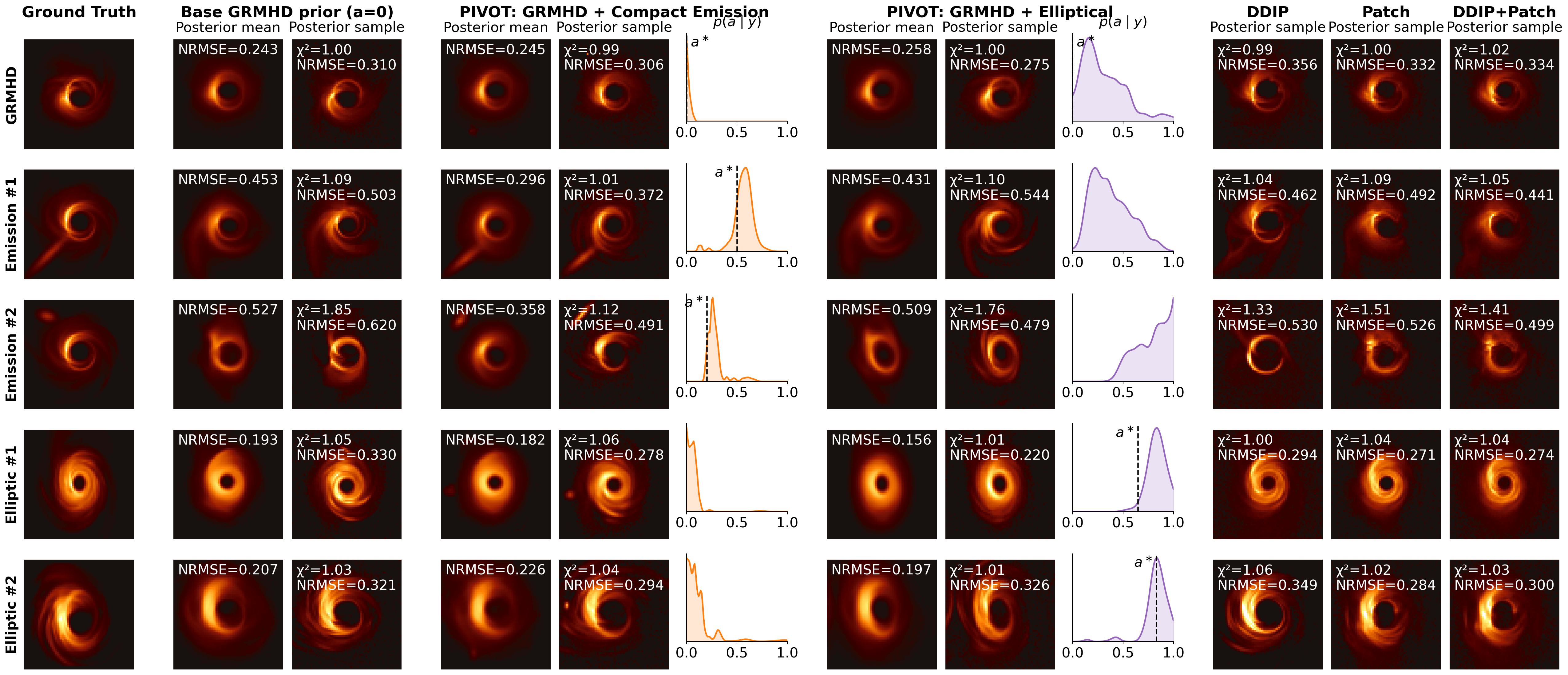}
    \caption{Posterior samples of synthetic in-distribution and OOD measurements under the base GRMHD prior (left), compact emission and ellipticity (middle) discrepancy priors, and baselines (right). The emission prior places substantial posterior mass near the true discrepancy $\va^*$ of the two jet observations, while identifying the elliptical observations as not discrepant. The elliptical prior gives uncertain results even for in-distribution GRMHD, indicating that the forward model is too noisy or ill-posed to constrain ellipticity. The baselines struggle to detect either discrepancy.}
    \label{fig:synthetic}
\end{figure}

\begin{figure}
    \centering
    \includegraphics[width=0.8\linewidth]{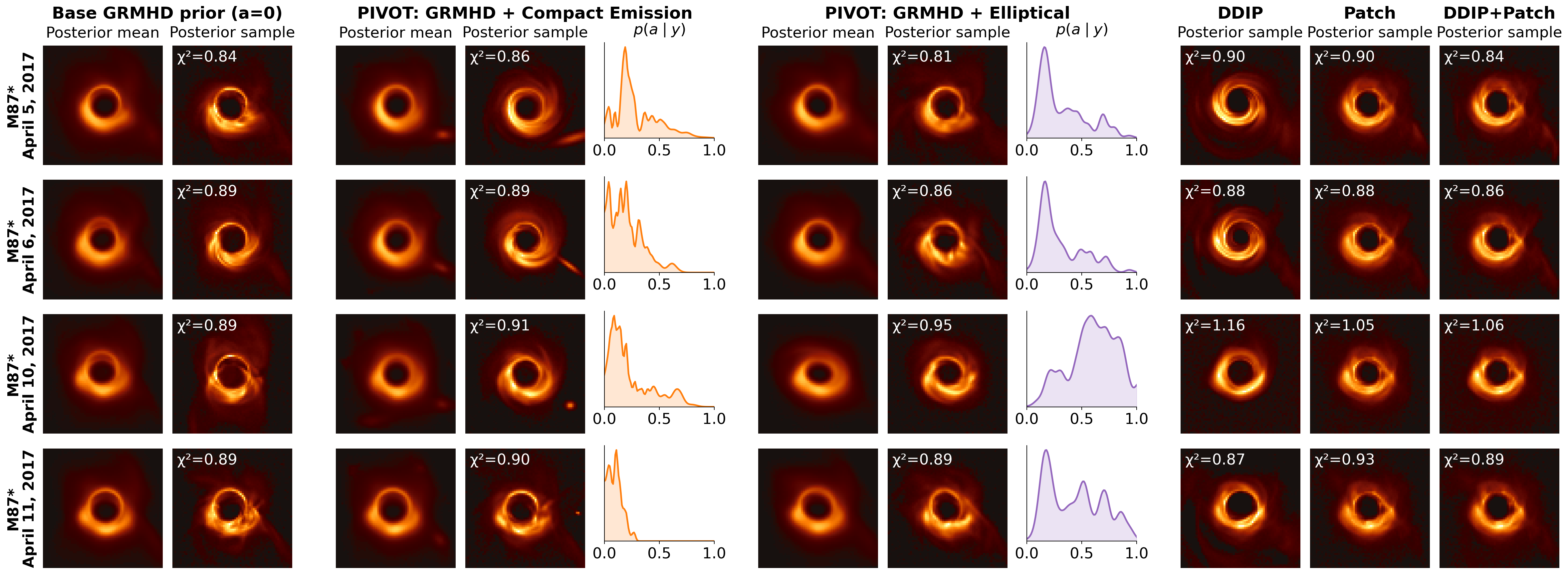}
    \caption{
    M87* posterior samples derived from real data collected across four different days in April 2017 under the base GRMHD prior (left), compact emission and ellipticity (middle) discrepancy priors, and baselines (right). The M87* has a non-trivial peak in the emission discrepancy posterior $p(\va \mid \vy)$. Furthermore, the posterior samples indicate that this emission exists only on the bottom right across observations from four different days, indicating that this is likely a feature supported by data rather than posterior uncertainty. On the other hand, because the forward model is not sensitive to ellipticity (Fig. \ref{fig:calibration}, bottom), $p(\va \mid \vy)$ is uncertain and any ellipticity conclusions are limited by the ill-posedness of the forward model.}
    \label{fig:m87}
\end{figure}
\vspace{-0.0in}

 We display the distribution of $\mathbb{E}[\va\mid\vy]$, or mean discrepancy supported by data, (Fig. \ref{fig:calibration}) for 64 in-distribution synthetic observations of GRMHD,  9 synthetic in-family observations across a wide range of discrepancy $\va^*$, 9 synthetic observations of the other perturbation family across a wide range of $\va^*$, and M87* observations using each day's forward models across all four observation days. We also show posterior samples for selected synthetic observations (Fig. \ref{fig:synthetic}) and M87* observations (Fig. \ref{fig:m87})  under the base GRMHD prior, expanded priors, and baselines; see Appendix \ref{sec:ddip-analysis} for more examples across a range of discrepancy strengths. For the emission discrepancy, the synthetic in-distribution GRMHD posterior and OOD family have low discrepancy, while the in-family posteriors recover the true discrepancy $\va^*$, indicating that the forward model is sensitive to extra emissions under all forward models from all four observation days (Fig. \ref{fig:failure-cases}, right). The M87* observations also show consistently higher emission discrepancy than most in-distribution observations, at percentiles $96.9$, $96.9$, $95.3$, $89.1$ respectively over the four days, consistent with an additional compact emission feature that could be associated with a jet. The few in-distribution samples that have higher discrepancy means than M87* have uncertain posteriors even under the base prior, resulting in emissions across the whole image as a result of ill-posedness (Appendix \ref{sec:extra-posterior-samples}). On the other hand, the M87* posterior is stable and emissions consistently appear at the bottom right of the image across all four days, with fading intensity over time. 
 
For the elliptical family, while in-family observations have high discrepancy, so do many in-distribution observations and emission family observations, indicating that the forward model is insensitive to ellipticity. The discrepancy posterior is broad and often overestimated at lower true discrepancies (Appendix \ref{sec:ddip-analysis}) and thus not as well suited for hypothesis testing (Fig. \ref{fig:failure-cases}, left/middle). Furthermore, the M87* elliptic discrepancies are not unusual compared to the ellipticity posterior of in-distribution GRMHD, at percentiles $51.6$, $50.0$, $87.5$, $73.4$ over the four days respectively. Thus, we can only conclude that the ellipticity of M87* is inconclusive with respect to GRMHD. For both perturbations, the adaptation baselines are not able to recover the missing features, as the posterior samples achieve reasonable datafit under the base prior; in Appendix \ref{sec:ddip-analysis} we find that more prominent features are necessary for accurate adaptation.

\vspace{-0.1in}





\section{Conclusion}

Rather than treating diffusion priors as fixed truths, our framework \methodname turns prior misspecification into a testable scientific question. By asking which plausible departures are supported by the measurements, it allows the data to determine how prior assumptions should be relaxed. We demonstrate this on controlled problems and on real black hole observations, where \methodname identifies features or latent variables supported by the data but missing from a narrow simulation prior. More broadly, this offers a path toward using learned priors not only to reconstruct the world, but to reveal where our models of it are incomplete.


\subsection*{AI use statement}


In this work, we used generative AI tools for [formulate mathematical claims, design or provide feedback on research  methodology or experiments, implement methods, clean and reformat dataset, support qualitative and thematic data analysis].
We have not used generative AI tools for [help develop theoretical models or conceptual frameworks, propose or refine hypotheses, assist with translation, or interpret results]
and [generating synthetic data sets,  provide critical ingredients for proving mathematical claims, assist in the writing of proofs] are not applicable to this work.
Additionally, we used generative AI tools for [create or modify scientific figures or images, create or edit software code, summarize or analyse existing literature, brainstorming, edit a research paper to improve readability, identify relevant literature]. We have reviewed all AI-assisted work. All generated math, text and code was manually verified by the authors. We take responsibility for the final content of this work,
including text, claims or artifacts produced with the aid of generative AI.

\subsection*{Ethics statement}


Our method could be used to identify sensitive characteristics from noisy observations, but we do not support that use case.

\subsection*{Reproducibility statement}

All code is attached in the supplementary material, all trained models will be released at publication, and all hyperparameters and experimental details are described in Appendix \ref{sec:experimental-details}. The OpenFWI dataset is publicly available \citep{deng2022openfwi}. The GRMHD dataset is not available for public use.




\section*{Acknowledgments}
This work was supported by NSF Award 2048237, an Amazon AI4Science Discovery Award, OpenAI, and a Allen and Charlotte Ginsburg Prize. FW is supported by a Kortschak Fellowship.

\bibliography{iclr2027_conference}
\bibliographystyle{iclr2027_conference}

\newpage
\appendix
\section{Appendix}

\subsection{Further validation of the ELBO and amortized mixture}\label{sec:elbo-validation}

As our method uses the ELBO approximation for evaluating density and learned amortization to represent the uniform mixture, we provide further tests in this section to determine whether these approximations cause issues in practice.

\subsubsection{ELBO}

We use the ELBO for efficiency and differentiability, allowing for MCMC sampling; for a low-dimensional grid-based evaluation of $p(a \mid y)$, we can evaluate the  exact density (up to discretization error) by integrating the Jacobian trace along the PF-ODE [5], providing a way to quantify ELBO approximation errors. We provide a comparison between the ELBO and the PF-ODE integrated density for 100 samples of the two GRMHD models below, across a range of PF-ODE discretization steps, and report the following results: (1) $\textbf{Mean error}$ to ground truth label $\|\mathbb{E}[a\mid x] - a^* \|$, (2) $W_1 \textbf{ distance}$ to the highly discretized PF-ODE reference $W_1( p(a\mid x), p_{\text{pf-ode 400}}(a\mid x)))$ and (3) computational efficiency.

\begin{table}[H]
\centering
\small
\setlength{\tabcolsep}{5pt}
\begin{tabular}{lrrrr}
\hline
Method & Mean error $\downarrow$ & $W_1$ to reference $\downarrow$ & Time (s) $\downarrow$ & Relative cost \\
\hline
ELBO (1 pair)      & $0.0134 \pm 0.0099$ & $0.0058 \pm 0.0070$ & 1.04  & $1.00\times$ \\
PF-ODE (25 steps)  & $0.1177 \pm 0.0672$ & $0.1163 \pm 0.0652$ & 1.69  & $1.62\times$ \\
PF-ODE (50 steps)  & $0.0246 \pm 0.0213$ & $0.0203 \pm 0.0200$ & 3.37  & $3.23\times$ \\
PF-ODE (100 steps) & $0.0146 \pm 0.0108$ & $0.0075 \pm 0.0088$ & 6.73  & $6.46\times$ \\
PF-ODE (200 steps) & $0.0149 \pm 0.0102$ & $0.0041 \pm 0.0055$ & 13.48 & $12.93\times$ \\
PF-ODE (400 steps) & $0.0140 \pm 0.0097$ & $0.0000 \pm 0.0000$ & 27.24 & $26.14\times$ \\
\hline
\end{tabular}
\caption{Comparison of ELBO and integrated PF-ODE density evaluation of $p(a \mid x)$ for the GRMHD compact emission prior. Mean error is with respect to the reference condition $|\mathbb{E}[a\mid x]-a^\star|$, and the $W_1$ distance is compared to the highly discretized 400-step PF-ODE reference. The ELBO provides an accurate conditional density estimate for significantly less computation.}
\label{tab:elbo_compact}
\end{table}

\begin{table}[H]
\centering
\small
\setlength{\tabcolsep}{5pt}
\begin{tabular}{lrrrr}
\hline
Method & Mean error $\downarrow$ & $W_1$ to reference $\downarrow$ & Time (s) $\downarrow$ & Relative cost \\
\hline
ELBO (1 pair)      & $0.0430 \pm 0.0318$ & $0.0092 \pm 0.0082$ & 1.04  & $1.00\times$ \\
PF-ODE (25 steps)  & $0.1209 \pm 0.1079$ & $0.1001 \pm 0.0958$ & 1.69  & $1.62\times$ \\
PF-ODE (50 steps)  & $0.0558 \pm 0.0418$ & $0.0246 \pm 0.0199$ & 3.38  & $3.23\times$ \\
PF-ODE (100 steps) & $0.0448 \pm 0.0341$ & $0.0101 \pm 0.0079$ & 6.74  & $6.45\times$ \\
PF-ODE (200 steps) & $0.0422 \pm 0.0334$ & $0.0049 \pm 0.0038$ & 13.52 & $12.94\times$ \\
PF-ODE (400 steps) & $0.0414 \pm 0.0331$ & $0.0000 \pm 0.0000$ & 27.34 & $26.17\times$ \\
\hline
\end{tabular}
\caption{Comparison of ELBO and integrated PF-ODE density evaluation of $p(a \mid x)$ for the GRMHD ellipticity prior. Mean error is with respect to the reference condition $|\mathbb{E}[a\mid x]-a^\star|$, and the $W_1$ distance is compared to the highly discretized 400-step PF-ODE reference. The ELBO provides an accurate conditional density estimate for significantly less computation.}
\label{tab:elbo_ellipticity}
\end{table}

Results are shown in Tables \ref{tab:elbo_ellipticity} and \ref{tab:elbo_compact}. The ELBO agrees with the highly discretized PF-ODE reference at significantly lower cost, and both give comparable mean error metrics; note that the non-zero MAE is because the ground truth condition does not necessarily need to be the posterior mean.

\subsubsection{Amortization}
We validate the amortized head in both (1) condition-space and (2) image-space. For (1), we evaluate the distribution $p(a \mid x)$ of 500 unconditional samples of $p(x \mid \varnothing)$ using the ELBO which we validated above; under ideal conditions, if the mixture is correctly represented then we have $\mathbb{E}_{x \sim p(x \mid \varnothing)} [p(a \mid x)] = \int p(a \mid x) p(x \mid \varnothing) dx =  p(a) = Unif[0,1]$,  the uniform distribution. The $W_1$ distance between them is reported in Table \ref{tab:mixture_w1}.

\begin{table}[H]
\centering
\normalsize
\begin{tabular}{lr}
\hline
Prior & $W_1\left(\mathbb{E}_{x\sim p(x\mid\varnothing)}[p(a\mid x)],\,\mathrm{Unif}[0,1]\right)$ \\
\hline
GRMHD + Ellipticity      & 0.0118 \\
GRMHD + Compact Emission & 0.0085 \\
FlatVelA + Chambers      & 0.0096 \\
\hline
\end{tabular}
\caption{$W_1$ distance between the estimated conditions $p(a \mid x)$ of the amortized head $p(x \mid \varnothing)$ and the ideal uniform distribution for all three priors. The averaged conditional posterior of the amortized head matches the uniform distribution closely for all three.}
\label{tab:mixture_w1}
\end{table}

For (2), we compare 2000 samples of $p(x \mid \varnothing)$ with 2000 samples generated by drawing $a \sim p(a)$ and then $x \sim p(x \mid a)$. We consider (1) uniform $p(a)$ (which is the ideal mixture distribution) and (2) truncated exponential and beta distributed $p(a)$ which are biased. We also compare FID with an independently drawn set of 2000 samples of $p(x \mid \varnothing)$ to determine the effect of finite sample bias in Table \ref{tab:mixture_fid}.

\begin{table}[H]
\centering
\tiny
\resizebox{\linewidth}{!}{%
\begin{tabular}{lr}
\hline
Comparison & FID $\downarrow$ \\
\hline
$p(x\mid\varnothing) \leftrightarrow p(x\mid\varnothing)$ & 3.323 \\
$p(x\mid\varnothing) \leftrightarrow \{a\sim\mathrm{Unif}[0,1],\,x\sim p(x\mid a)\}$ & 3.361 \\
$p(x\mid\varnothing) \leftrightarrow \{a\sim\mathrm{Exp}(3),\,x\sim p(x\mid a)\}$ & 14.015 \\
$p(x\mid\varnothing) \leftrightarrow \{a\sim\mathrm{Beta}(2,1),\,x\sim p(x\mid a)\}$ & 7.086 \\
$\{a\sim\mathrm{Unif}[0,1],\,x\sim p(x\mid a)\} \leftrightarrow \{a\sim\mathrm{Exp}(3),\,x\sim p(x\mid a)\}$ & 12.060 \\
$\{a\sim\mathrm{Unif}[0,1],\,x\sim p(x\mid a)\} \leftrightarrow \{a\sim\mathrm{Beta}(2,1),\,x\sim p(x\mid a)\}$ & 8.542 \\
$\{a\sim\mathrm{Exp}(3),\,x\sim p(x\mid a)\} \leftrightarrow \{a\sim\mathrm{Beta}(2,1),\,x\sim p(x\mid a)\}$ & 29.202 \\
\hline
\end{tabular}%
}
\caption{Comparison of image-space FID between the amortized head $p(x \mid \varnothing)$ and mixtures generated by drawing $a \sim p(a)$, then $x \sim p(x \mid a)$ under different hyperpriors $p(a)$. The first row, which compares the FID to an independently drawn batch of the same distribution, quantifies the finite-sample variability. The FID between the amortized and true uniform mixture is comparable to the finite-sample variability.}
\label{tab:mixture_fid}
\end{table}

The amortized condition distribution closely matches the uniform distribution, and the image distribution  agrees with the true mixture distribution (up to finite sample bias, as demonstrated by the FID between $p(x \mid \varnothing)$ and itself).

\newpage

\subsection{Why tempering diffusion priors is problematic} \label{sec:tempering-bad}

Tempering the prior has been a common nonparametric approach for weakening prior assumptions  for both classical Bayesian analysis and deep learning \citep{kapoor2022uncertainty}. However, for high-dimensional image priors, tempering changes the expected energy distribution of its samples. As an example, for the standard multivariate normal $p(\vx) = \mathcal{N}(0,\mI)$ in $\mathbb{R}^d$, its samples $\vx$ have expected energy $\mathbb{E}\|\vx\|^2 = d$; while the tempered $p(\vx)^{0.5} \propto \mathcal{N}(0,2\mI)$ has samples $\vx$ with expected  energy $\mathbb{E}\|\vx\|^2 = 2d$. 

To empirically demonstrate the problem on diffusion priors, we use annealed MCMC \citep{du2023reduce} to sample from the tempered GRMHD diffusion prior. See Fig. \ref{fig:tempering} for a visual example. The energy distribution changes and the samples end up all concentrated at increasingly high energy as the tempering exponent $\beta$ decreases, rather than becoming more general prior as expected.

\begin{figure}[H]
    \centering
    \includegraphics[width=0.95\linewidth]{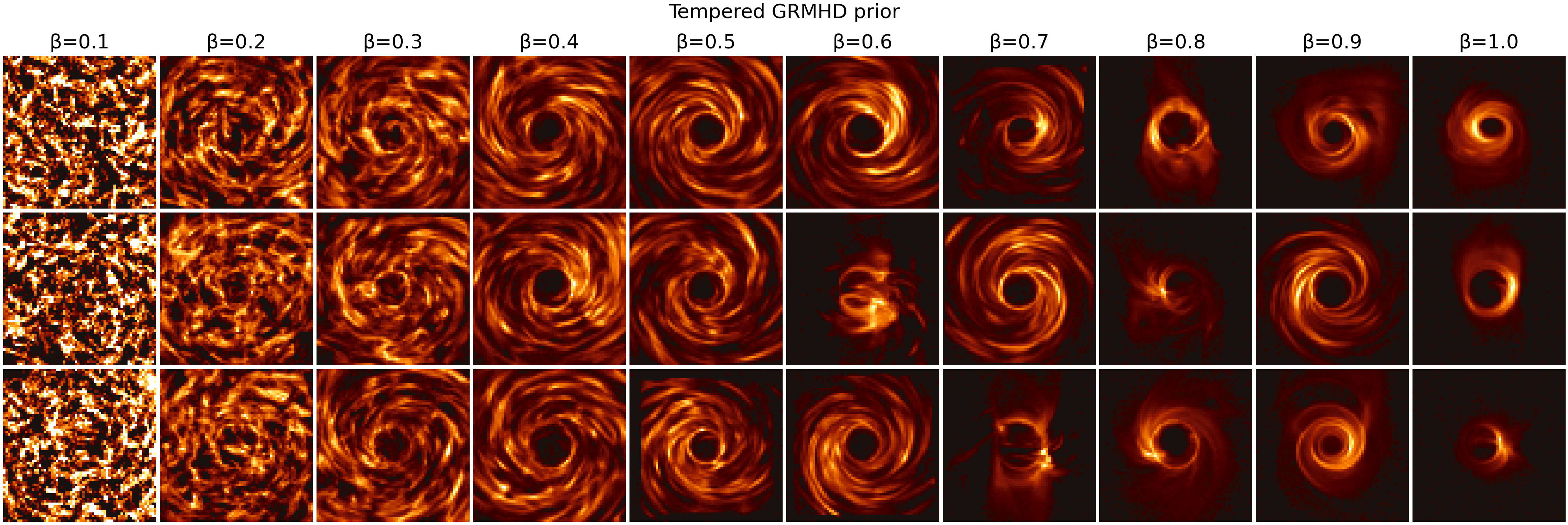}
    \caption{Samples of the tempered GRMHD prior, $p(\vx)^\beta$ with $\beta \in [0.1, 1.0]$. }
    \label{fig:tempering}
\end{figure}
\newpage
\subsection{Gaussian Ablation on 2D Discrepancy}\label{sec:gaussian-ablation-2d}
In this section, we perform a Gaussian ablation but with a 2D discrepancy $\va$, demonstrating that \methodname can be used to jointly infer multi-dimensional discrepancy hypotheses. We use the same multimodal Gaussian mixture prior, linear inverse problem, and mean-shift perturbation setup described in \ref{sec:gaussian-ablation}. We also still use the analytic diffusion score, DPMC sampling, and ELBO evaluation to match real use cases. With $\va \in [0,1]^2$, we consider three cases: (1) the discrepancy is completely orthogonal to the observations (orthogonal), (2) one component of the discrepancy is visible to the observations (Weak), and (3) both components are visible to the observations (Strong). 

Results are displayed in Fig. \ref{fig:gaussian-ablation-2d}. More posterior samples are required to reduce the finite-sample Monte Carlo error, with the orthogonal case remaining the most difficult similar to the 1D result. Furthermore, MCMC becomes more compute-efficient than grid-based evaluation.

\begin{figure}[H]
    \centering
    \includegraphics[width=1.0\linewidth]{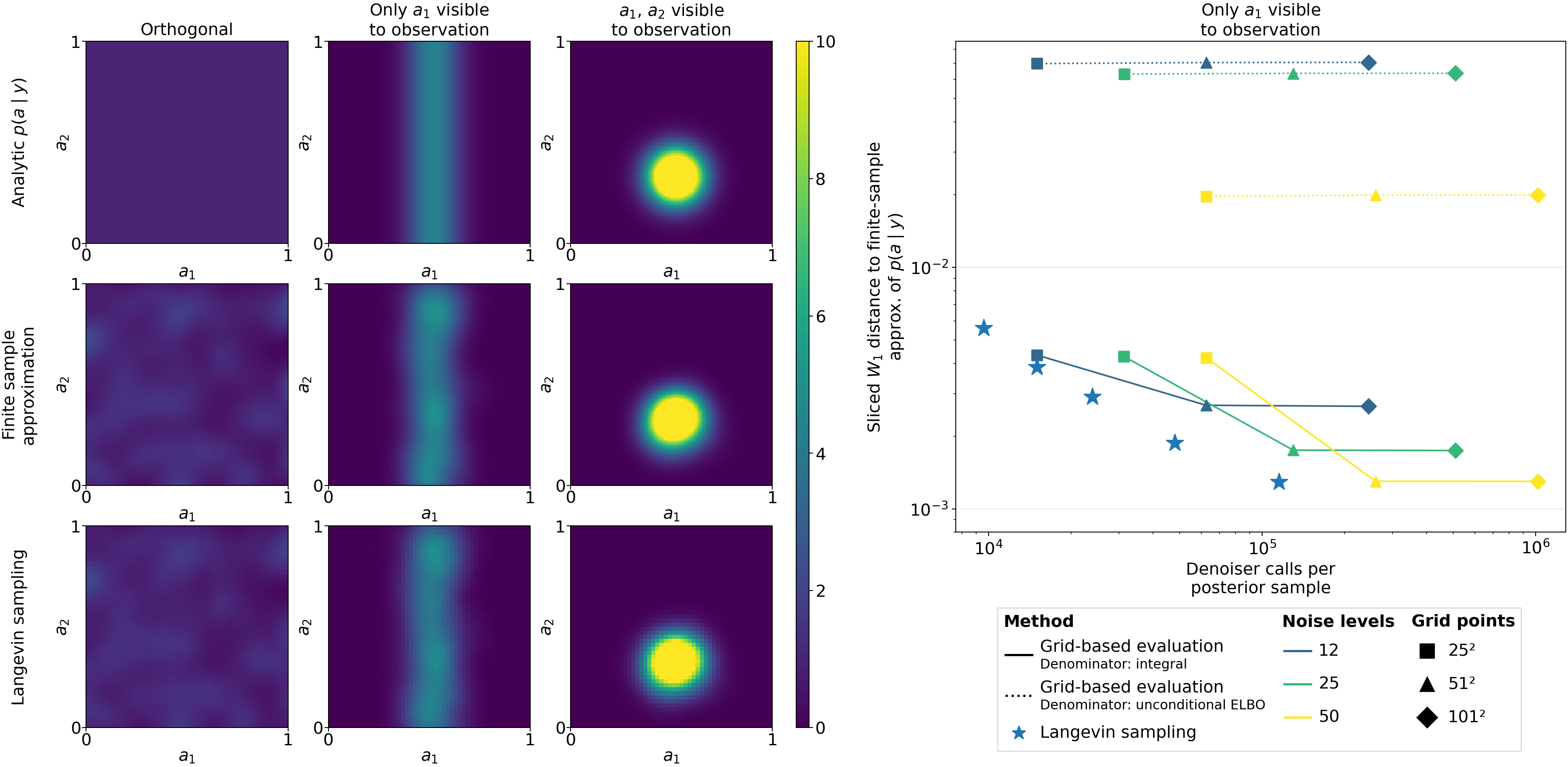}
    \caption{\textbf{(Left)} Discrepancy posterior $p(\va \mid \vy)$ results on a Gaussian toy experiment where the forward operator has an increasing level of sensitivity to the perturbation (orthogonal, weak, strong) and $\va \in [0,1]^2$. The ideal analytic $p(\va \mid \vy)$,  analytic finite-sample (50) approximation of $p(\va \mid \vy)$, and Langevin sampled discrepancy are shown. The density shape is correctly recovered up to the Monte Carlo error of the finite-sample approximation despite the approximate sampler and ELBO. In comparison to the 1D $\va$ case, the finite-sample Monte Carlo error is higher. \textbf{(Right)} computation vs. accuracy plot for the partial sensitivity case. In comparison to the 1D $\va$ case, MCMC becomes a more efficient alternative compared to grid-based integration.}
    \label{fig:gaussian-ablation-2d}
\end{figure}

\newpage
\subsection{Selected Posteriors for Black Hole Imaging}\label{sec:extra-posterior-samples}

In Section \ref{sec:black-hole-imaging}, we show that on April 5 and 6, 2017, the M87* observations have higher compact emission discrepancy than in-distribution synthetic observations, except for two GRMHD observations (which were the same two GRMHD samples on both days). We display the posterior mean, standard deviation, and 10 randomly selected posterior samples for these two in-distribution GRMHD under both the base GRMHD prior as well as the expanded prior with emissions in Fig. \ref{fig:day096-grmhd}. We also display the posterior mean, standard deviation, and 10 randomly selected posterior samples of M87* on April 5 and 6 in Fig. \ref{fig:day096-m87}.

\begin{figure}[H]
    \centering
    \includegraphics[width=0.95\linewidth]{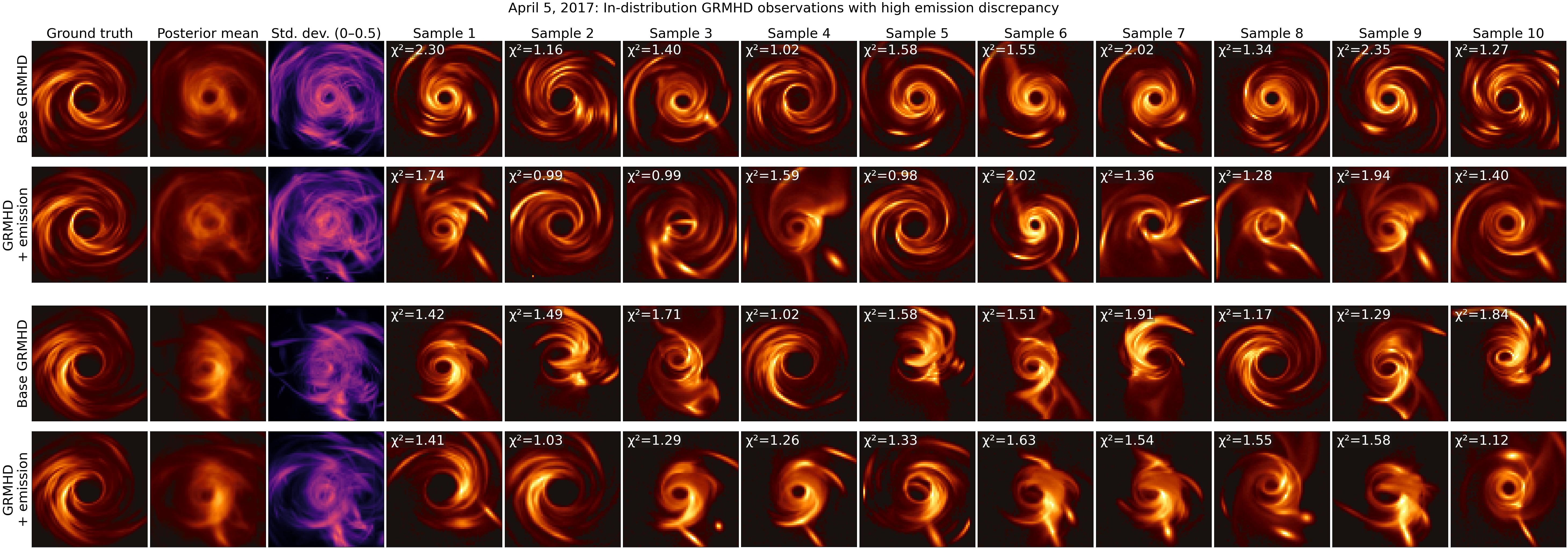}
    \includegraphics[width=0.95\linewidth]{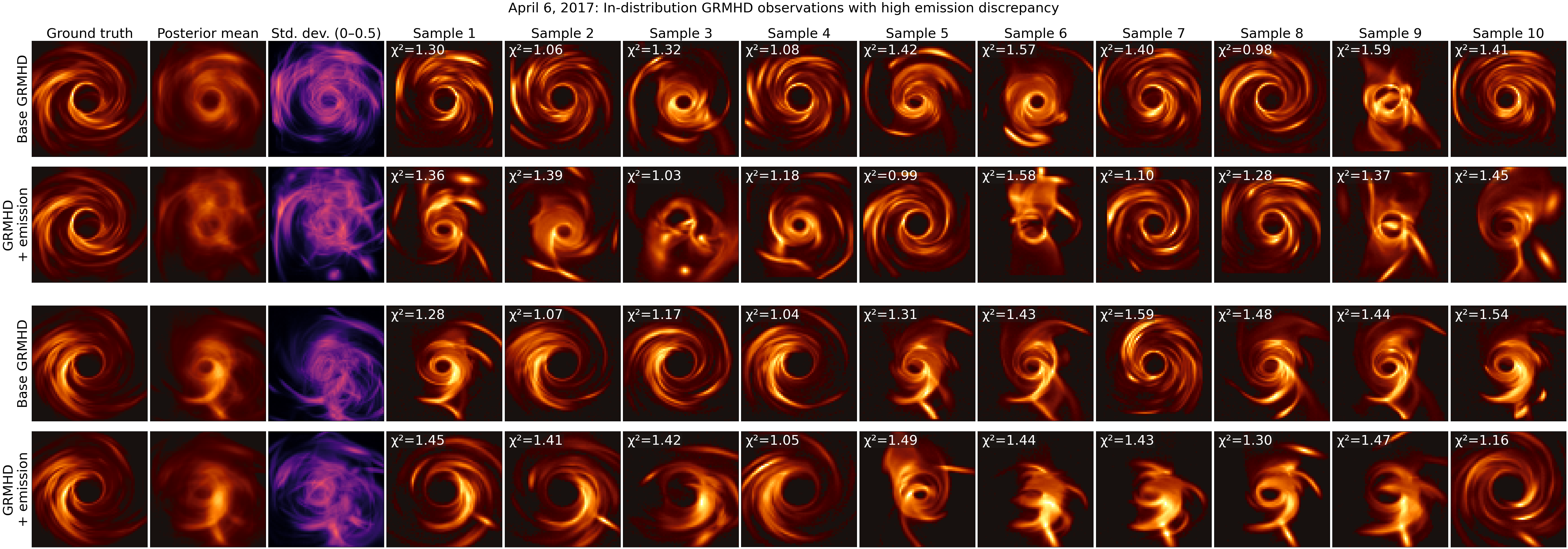}
    \caption{Posterior mean, standard deviation, and samples of the two GRMHD with high emission discrepancy, using forward operators from April 5/6 and both the base GRMHD prior and the emission expanded prior. The posterior is highly multimodal and uncertain, resulting in random emission-like features all over the image even under the base prior.}
    \label{fig:day096-grmhd}
\end{figure}

\begin{figure}[H]
    \centering
    \includegraphics[width=0.95\linewidth]{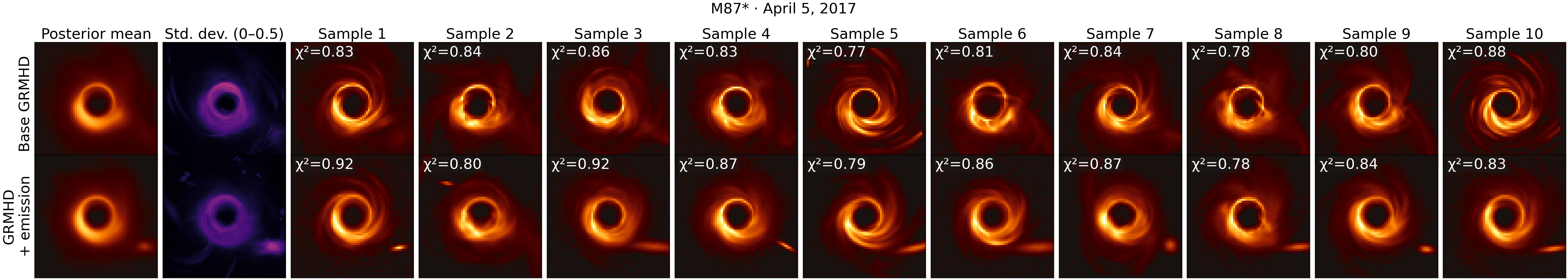}
    \includegraphics[width=0.95\linewidth]{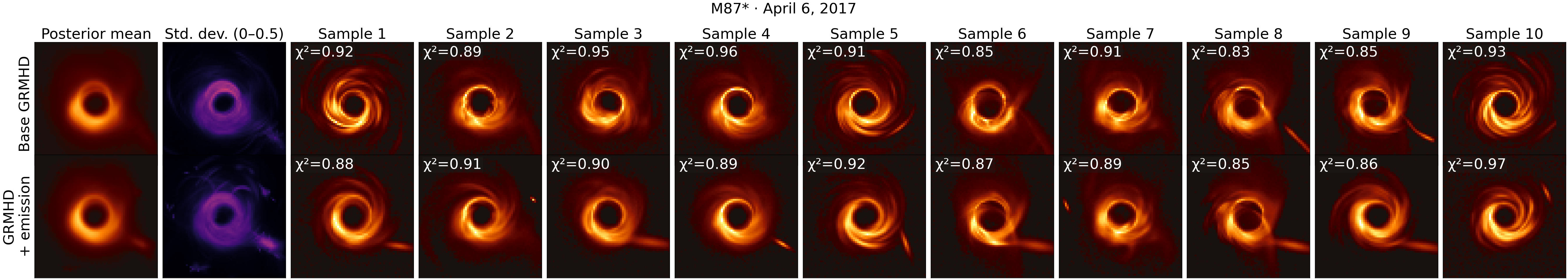}
    \caption{Posterior mean, standard deviation, and samples of M87* from April 5/6 under both the base GRMHD prior and the emission expanded prior. The posterior is concentrated, with much lower variance and the emission is almost always located in the bottom-right.}
    \label{fig:day096-m87}
\end{figure}

\newpage
\subsection{Sensitivity under Measurement Noise and Perturbation Strengths} \label{sec:ddip-analysis}

In this section, we perform an ablation over decreasing measurement noise levels in the travel time tomography in-family example (Fig. \ref{fig:ddip-baseline-ttt}), where we observe that there is much lower discrepancy posterior spread, or uncertainty, at lower measurement noise levels. We also demonstrate how increasing perturbation strength (feature prominence) in black hole emissions (Fig. \ref{fig:ddip-baseline-bh-1} and \ref{fig:ddip-baseline-bh-2}) and ellipticity  (Fig. \ref{fig:ddip-baseline-bh-3} and \ref{fig:ddip-baseline-bh-4}) affects the estimate. For the emission discrepancy, \methodname is able to infer posteriors that assign large mass to the true discrepancy $\va^*$ across a large range of measurement noise or feature prominence, while the elliptical discrepancy consistently overestimates the true discrepancy due to confounding effects in the forward model.

We also demonstrate that DDIP methods perform the best when there is significant conflict between the observations and the prior; this can be achieved by having low observation noise, a less ill-posed observation, or a stronger mismatched feature. In the travel time tomography example, decreasing the measurement noise $\sigma_y$ results in monotonically improved reconstruction of the subsurface chamber. For the black hole imaging example, we see that increasing the feature's prominence can result in more accurate adaptations, such as the elongated photon ring in Fig. \ref{fig:ddip-baseline-bh-4}, bottom. However, in the second example, for the adaptation to be successful the feature has to be strong enough that it is even recovered under the base GRMHD prior, as shown in the bottom row of Fig. \ref{fig:ddip-baseline-bh-2}. 

\begin{figure}[H]
    \centering
    \includegraphics[width=0.75\linewidth]{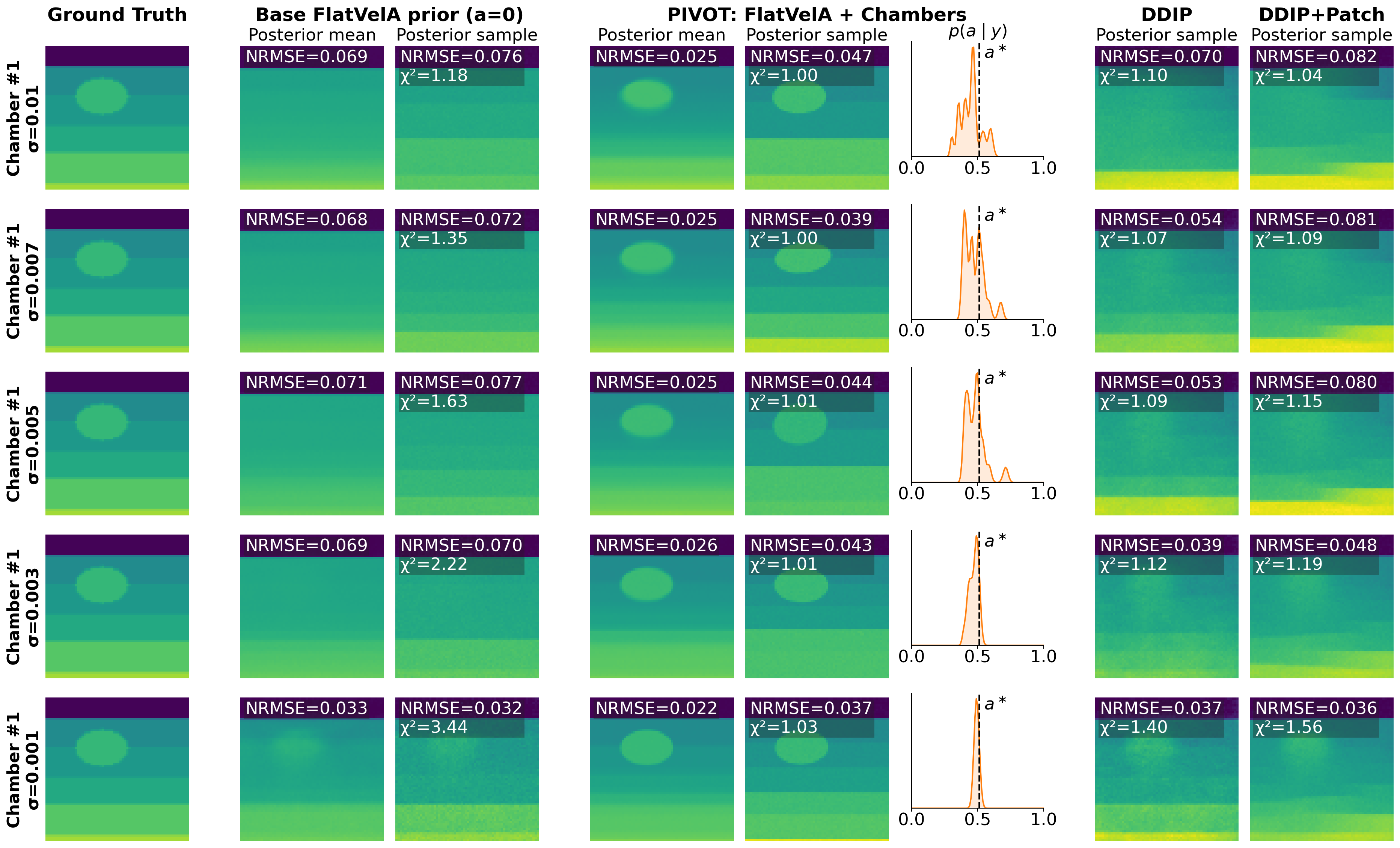}
    \caption{Travel Time Tomography results with decreasing observation noise. Lower noise results in a more accurate adaptation as well as lower uncertainty in the discrepancy posterior $p(a \mid y)$.}
    \label{fig:ddip-baseline-ttt}
\end{figure}

\begin{figure}[H]
    \centering
    \includegraphics[width=1.0\linewidth]{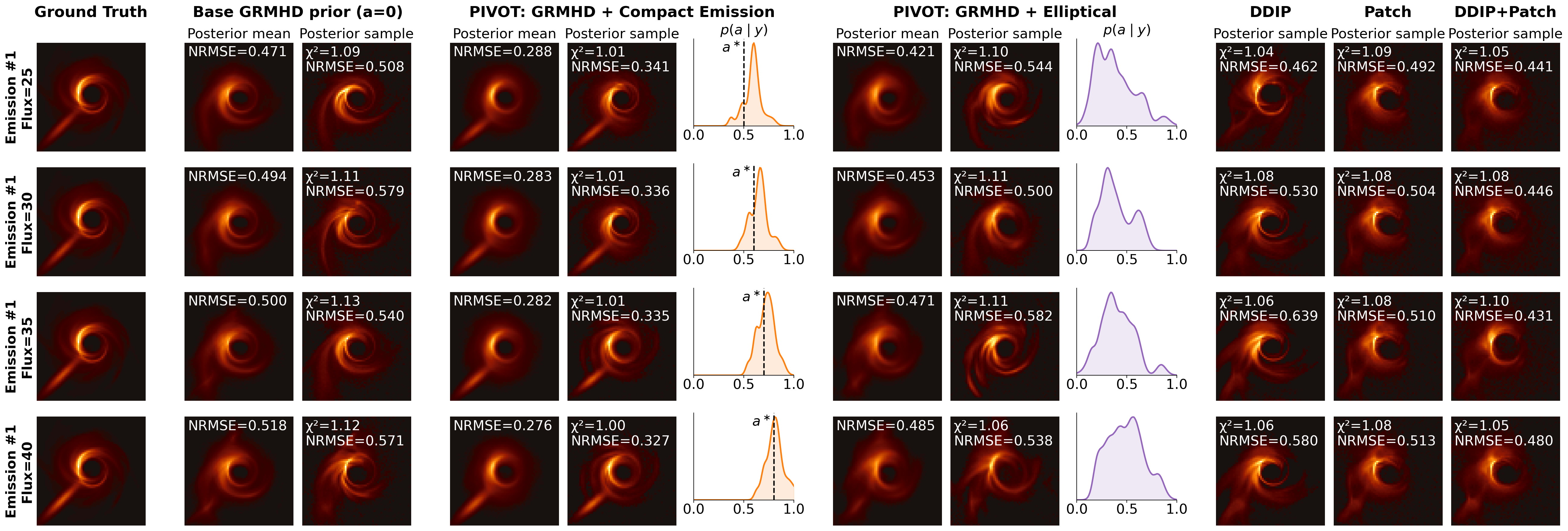}
    \caption{Black hole imaging results with increasing emission feature prominence, in flux. A more prominent feature results in a clearer emission in the adaptation, while the emission discrepancy posterior correctly tracks the ground truth discrepancy $\va^*$, and the elliptic discrepancy remains low.}
    \label{fig:ddip-baseline-bh-1}
\end{figure}

\begin{figure}[H]
    \centering
    \includegraphics[width=1.0\linewidth]{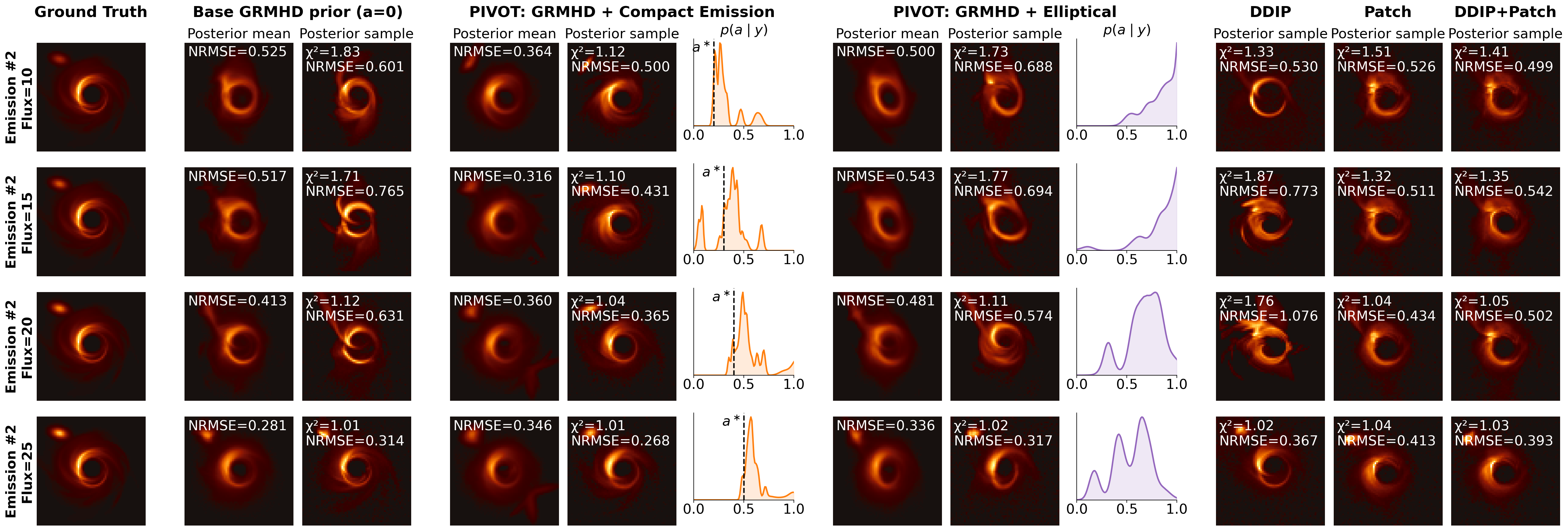}
    \caption{Black hole imaging results with increasing emission feature prominence, in flux. The adaptation is only successful when the feature is so strong that it gets reconstructed under the base GRMHD prior. Meanwhile, the emission discrepancy posterior correctly tracks the ground truth discrepancy $\va^*$, and the elliptic discrepancy remains low.}
    \label{fig:ddip-baseline-bh-2}
\end{figure}

\begin{figure}[H]
    \centering
    \includegraphics[width=1.0\linewidth]{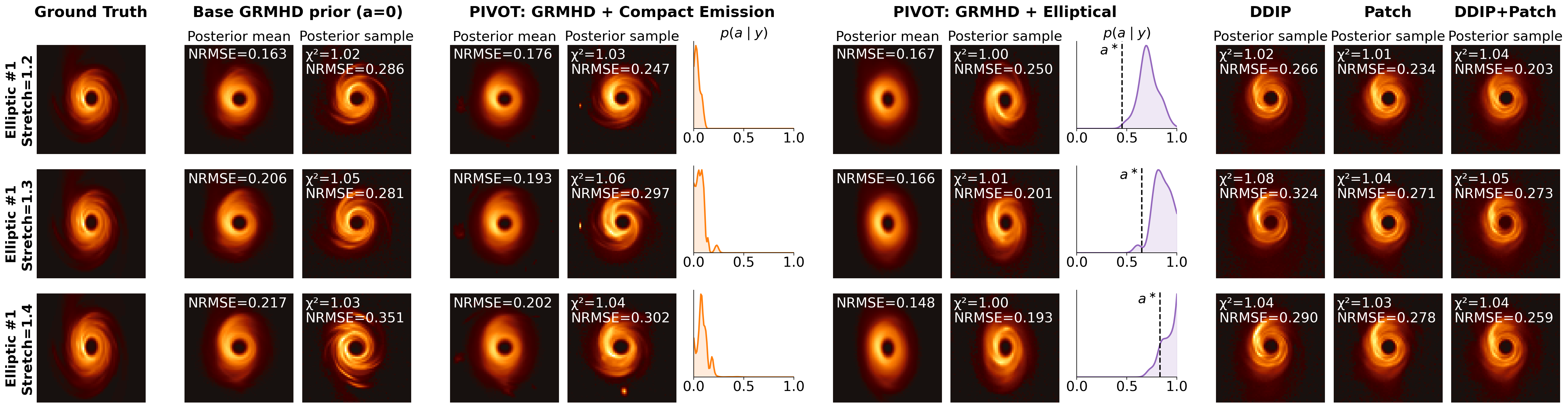}
    \caption{Black hole imaging results with increasing ellipticity. The adaptation is not successful and the photon ring remains circular. Meanwhile, the emission discrepancy posterior correctly determines that the mismatch is not well explained by additive compact emission, and the elliptic discrepancy is overestimated at lower discrepancies  compared to the true $a^*$ due to the limitations of the forward model.}
    \label{fig:ddip-baseline-bh-3}
\end{figure}

\begin{figure}[H]
    \centering
    \includegraphics[width=1.0\linewidth]{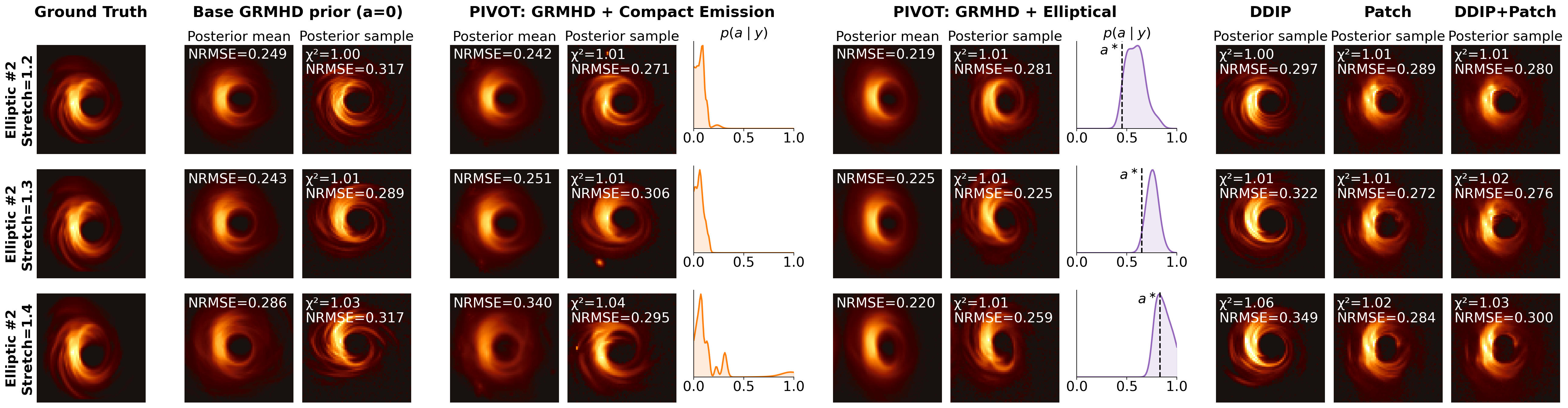}
    \caption{Black hole imaging results with increasing ellipticity. For this example, both the patch prior and the patch adaptation (DDIP+Patch) are successful only at the high 1.4x stretch, and the photon ring is elongated. Meanwhile, the emission discrepancy posterior correctly determines that the mismatch is not well explained by additive compact emission, and the elliptic discrepancy is overestimated at lower discrepancies compared to the true $a^*$ due to the limitations of the forward model.}
    \label{fig:ddip-baseline-bh-4}
\end{figure}

\newpage
\subsection{Experimental Details}\label{sec:experimental-details}

\subsubsection{Hyperparameters}

To train our conditional diffusion models, we used a SongUNet with standard EDM \citep{karras2022elucidating} preconditioning and noise sampling. The OpenFWI priors were resized to $64\times64$, and the GRMHD prior is resized to $64\times64$ and then augmented with random zoom between $[0.75, 1.25]$. Both the travel time tomography and black hole imaging diffusion models used 128 channels, channel mult $[1,2,2]$, 1 residual block, and $16\times16$ attention. We use $p_{dropout} = 0.2$ and train for $50000$ optimizer steps, which took 2.5 hours per model on an NVIDIA A100. 

During inference, we used DPMC \citep{zhu2409think} for posterior sampling, which performs annealed MCMC along $p(\vx_t) p(\vy \mid \mathbb{E}[\vx_0 \mid \vx_t])^{\beta_t}$. For all methods, we use $\sigma_{max} = 50, \sigma_{min}=0.01$, and the Karras $\rho=7$ noise schedule for sampling. We also parameterize $\beta_t = C / (C + \sigma_t^2)$ as a surrogate modeling the uncertainty of $p(\vx_0 \mid \vx_t)$ that is missing from using the point estimate $\mathbb{E}[\vx_0 \mid \vx_t]$. We sample $25$ posterior samples for each observation, and for ELBO computation, we use $101$ gridpoints and $25$ noise levels with $\sigma_{max}=5$. We list the shared hyperparameters and runtime on one NVIDIA A100 here:

\begin{table}[h!]
\centering
\begin{tabular}{|c|c|c|c|}
\hline
\textbf{Hyperparameter} & \textbf{Gaussian} & \textbf{Seismic Imaging} & \textbf{Black Hole Imaging} \\
\hline
Annealing steps  & 200 & 200 & 300 \\
\hline
Corrector steps  & 20 & 20 & 20 \\
\hline
Corrector LR  & $0.3 \sigma_t^2$ & $0.005\sigma_t^2$ & $0.05 \sigma_t^2$ \\
\hline
$C$  & 0.1 & 10 & 0.02 \\
\hline
Sampling $x \mid y$ runtime (25 samples) & $<$0.1 min & 9.3 min & 6.2 min \\
\hline
ELBO grid $a \mid y$ runtime (25 samples) & $<$0.1 min & 5.3 min & 5.3 min \\
\hline
\end{tabular}
\label{tab:example}
\end{table}

The Gaussian ablation MCMC used Langevin dynamics lr=$0.0005$ with 200 burn-in steps.

We tuned DDIP and DDIP+Patch independently on synthetic observations for each task, but we found that the same task-specific parameters work well with both methods. We chose to finetune all network weights during adaptation which has been demonstrated to perform better \citep{hu2024patch}. For patch-based methods, we choose a random tiling offset at each sampling step as is done in \cite{hu2024learning}. We report tuned hyperparameters and runtimes on one A100 below:

\begin{table}[h!]
\centering
\begin{tabular}{|c|c|c|}
\hline
\textbf{Hyperparameter} & \textbf{Seismic Imaging} & \textbf{Black Hole Imaging} \\
\hline
Network LR  & $0.0001$ & $0.00003$ \\
\hline
Network updates  & $1$ & $3$ \\
\hline
Reconstruction likelihood LR  & $0.01$ & $0.01$ \\
\hline
Reconstruction likelihood updates & $3$ & $3$ \\
\hline
Adaptation count & $10$ & $30$ \\
\hline
DDIP runtime (1 recon) & 4.4 min & 2.7 min \\
\hline
DDIP+Patch patch size & $16$ & $16$ \\
\hline
DDIP+Patch runtime (1 recon) & 4.4 min & 2.8 min \\
\hline
\end{tabular}
\label{tab:example}
\end{table}

Finally, in black hole imaging, closure quantities remove absolute phase information, the reconstruction is invariant to image space translations. Therefore, we align posterior samples via normalized cross-correlation before computing reconstruction metrics or posterior mean / standard deviation. We do use the potentially offset original images for ELBO computation, for consistency with the diffusion model's distribution.
\newpage
\subsubsection{Discrepancy parameterizations}

\begin{figure}[H]
        \centering
        \includegraphics[width=0.3\linewidth]{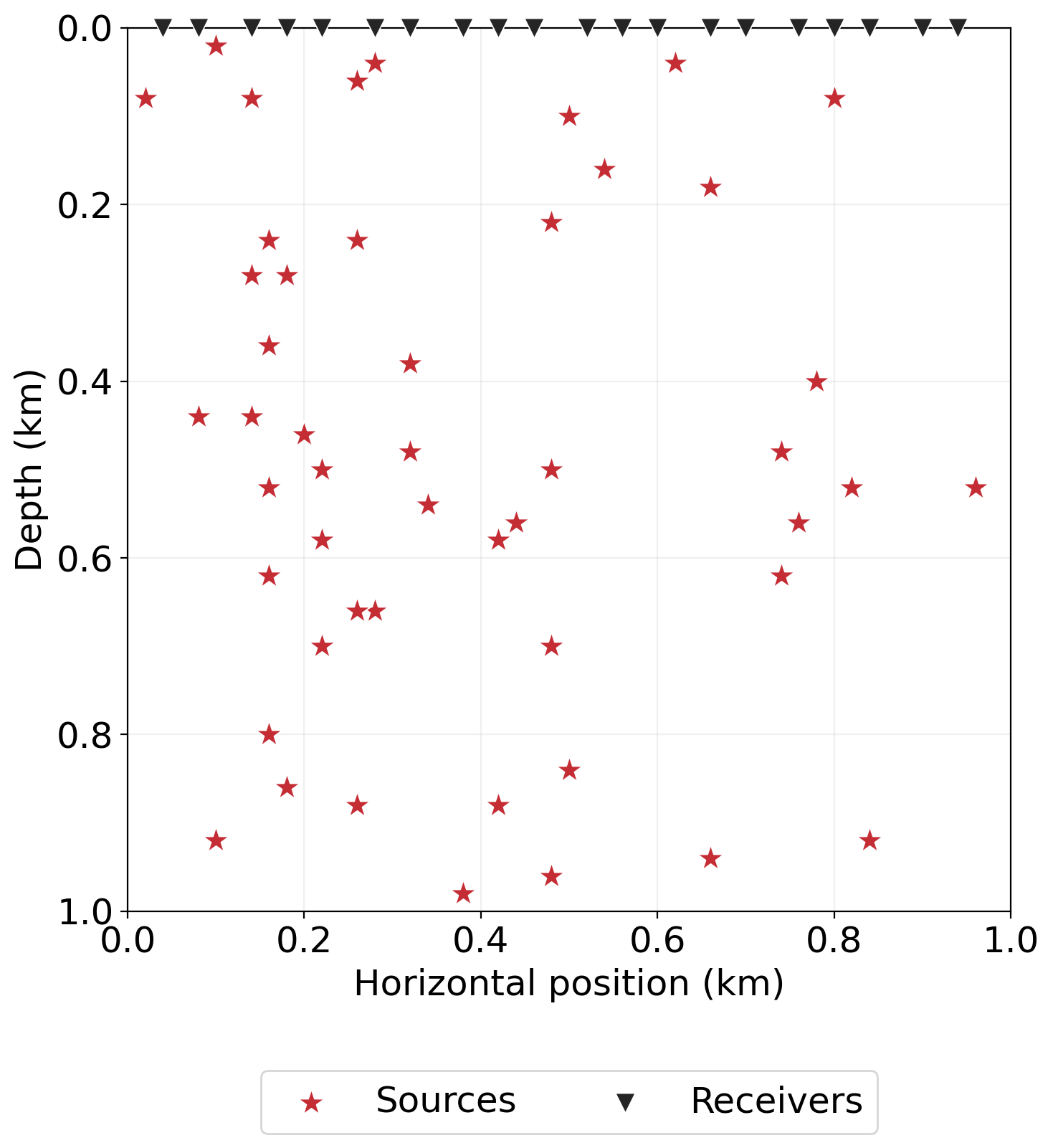}
        \caption{Locations of the 49 earthquake sources and 20 receivers used in the synthetic Travel Time Tomography experiments.}
        \label{fig:ttt-sources}
    \end{figure}

\begin{figure}[H]
    \centering
    \includegraphics[width=0.4\linewidth]{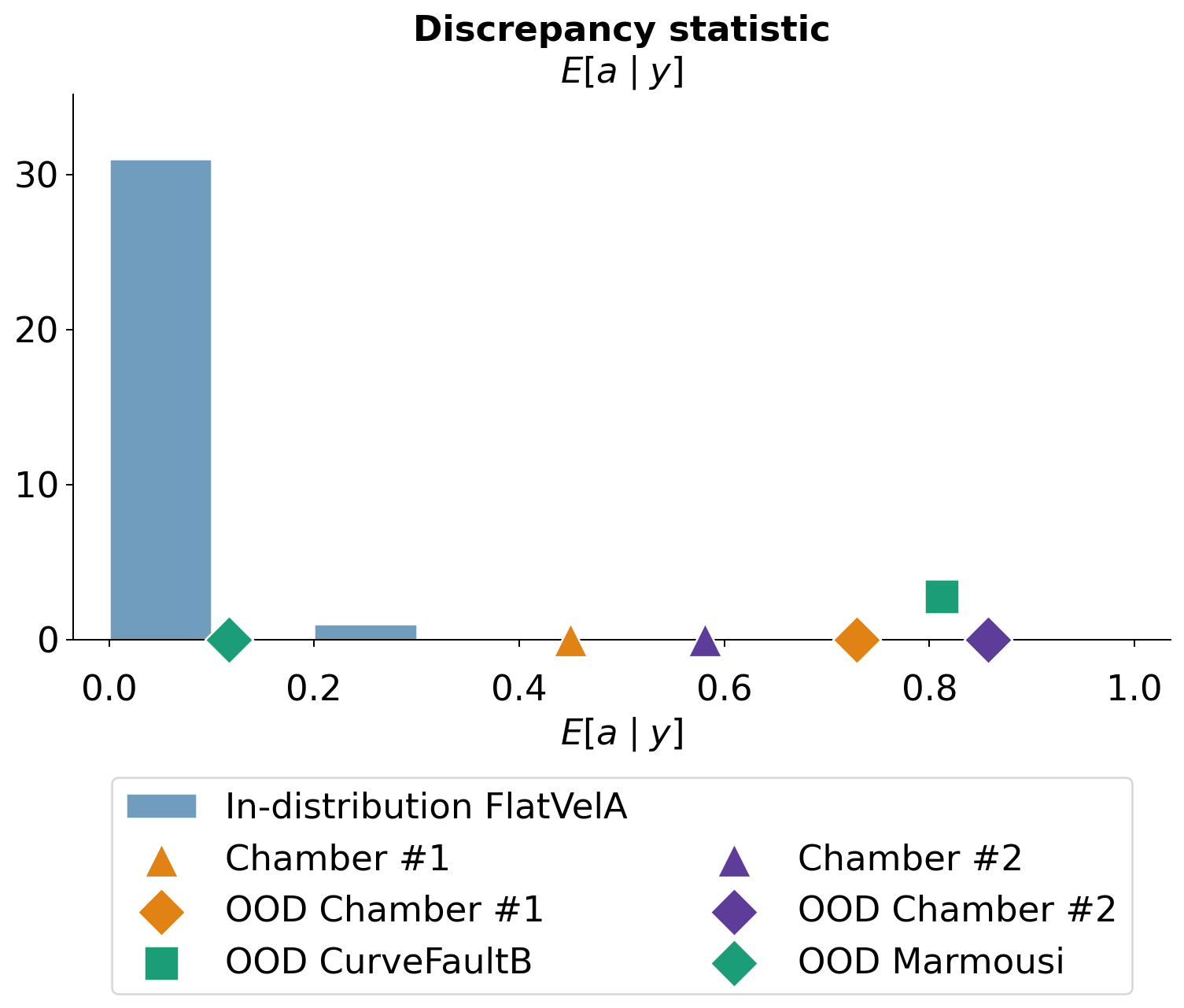}
        \caption{Mean discrepancy $\mathbb{E}[\va \mid \vy]$ for in-distribution observations $\vy$, in-family synthetic chamber observations $\vy$, and OOD observations $\vy$. The forward model is sensitive to the discrepancy.}
    \label{fig:ttt-calibration}
\end{figure}

For the travel time tomography experiment, we simulate $20$ receivers and $49$ earthquake sources in a 1 km x 1 km area with observation noise of $\sigma=0.01$ seconds, and use factored fast marching to compute likelihood gradients: the same setup described in \cite{gao2021deepgem}; see Fig. \ref{fig:ttt-sources}. We parameterize our chamber perturbation $\va \in [0,1]$ as an ellipse with uniform random center, axis ratio sampled uniformly from $[0.4, 1.0]$, velocity uniformly between at least $0.5$ km/s higher than the background and $4.5$ km/s, and the conditioning $\va\in[0,1]$ represents the area $600\va$ pixels$^2$. We also impose a uniform $p(\va)$. This perturbation is well-detected by observations; see Fig. \ref{fig:ttt-calibration}.

For black hole imaging, our emission discrepancy family consists of off-ring Gaussians with centers located between a radial distance of $25$ and $30$ pixels from the image center, axis ratio sampled uniformly from $[0.2, 1.0]$, rotation sampled uniformly from $[0, 2\pi]$, peak amplitude value sampled uniformly from $[0.5, 1.0]$, and the condition $\va \in [0,1]$ represents added flux of $50a$ in normalized diffusion coordinates. The elliptical family maps $\va\in[0,1]$ to an area-preserving anisotropic deformation of ratio $1.5^\va$ along a random axis.  We also impose a uniform $p(\va)$ for both families.

\end{document}